\documentclass[10pt]{article}
\usepackage{amsmath,amsthm,amsfonts,amssymb,latexsym,graphicx,stmaryrd,mathtools}
\SetSymbolFont{stmry}{bold}{U}{stmry}{m}{n}
\usepackage{algorithm}
\usepackage{algpseudocode}
\usepackage[pdfpagemode=UseNone,pdfstartview=FitH]{hyperref}
\newcommand{\Extra}[1]{}

\newcommand{\OCMVI}{Vovk:2015-cross}
\newcommand{\OCMVII}{Vovk/Petej:2014UAI}

\newcommand{\OCMXIII}{Vovk/etal:2015NIPS}

\newcommand{\OCMXVII}{Vovk/etal:2019ML}
\newcommand{\OCMXVIII}{Vovk:2019COPA}
\newcommand{\OCMXIX}{Vovk/Bendtsen:2018}
\newcommand{\OCMXX}{Vovk/etal:2018Braverman}
\newcommand{\OCMXXI}{Vovk:arXiv1212}

\newcommand{\OCMXXIII}{Vovk/etal:arXiv1902}

\newcommand{\AI}{\\\hspace*{3cm}}

\allowdisplaybreaks[4]

\newtheorem{theorem}{Theorem}

\newtheorem{proposition}[theorem]{Proposition}
\theoremstyle{definition}
\newtheorem{definition}{Definition}
\newtheorem{example}[theorem]{Example}
\newtheorem{remark}{Remark}

\DeclareMathOperator{\Prob}{\mathbb{P}}
\DeclareMathOperator{\Expect}{\mathbb{E}}
\DeclareMathOperator{\conv}{conv}
\DeclareMathOperator{\CRPS}{CRPS}
\DeclareMathOperator{\II}{\boldsymbol{1}}
\newcommand{\R}{\mathbb{R}}
\newcommand*{\dd}{\,\mathrm{d}}

\title{Computationally efficient versions of conformal predictive distributions}
\author{Vladimir Vovk, Ivan Petej, Ilia Nouretdinov,\\Valery Manokhin, and Alex Gammerman}

\begin{document}
\maketitle

\begin{abstract}
  Conformal predictive systems are a recent modification of conformal predictors
  that output, in regression problems, probability distributions for labels of test observations
  rather than set predictions.
  The extra information provided by conformal predictive systems may be useful, e.g.,
  in decision making problems.
  Conformal predictive systems inherit the relative computational inefficiency of conformal predictors.
  In this paper we discuss two computationally efficient versions of conformal predictive systems,
  which we call split conformal predictive systems and cross-conformal predictive systems.
  The main advantage of split conformal predictive systems
  is their guaranteed validity,
  whereas for cross-conformal predictive systems
  validity only holds empirically and in the absence of excessive randomization.
  The main advantage of cross-conformal predictive systems
  is their greater predictive efficiency.
\end{abstract}

\section{Introduction}

Two sister methods that have been widely presented at the COPA series of workshops
are conformal prediction and Venn prediction.
Both methods enjoy provable properties of validity under the IID model
but their outputs are very different:
whereas Venn predictors output probabilities
(more precisely, upper and lower probabilities),
conformal predictors output p-values
(often packaged as prediction sets).
Not only the outputs but also the areas of application are different for the two methods:
Venn predictors are only applicable to classification problems,
whereas conformal predictors are applicable to both classification and regression.

A recent development in conformal prediction has been
the definition and study of conformal predictive systems
(CPS, which we use for both singular and plural)
in \cite{\OCMXVII},
based on the parallel work on predictive distributions in parametric statistics
(see, e.g., \cite[Chapter~12]{Schweder/Hjort:2016} and \cite{Shen/etal:2017}).
In the case of regression problems, CPS output predictive distributions;
the difference between p-values and probabilities is often emphasized in statistics,
but in the case of CPS the p-values get arranged into a probability distribution
thus essentially becoming probabilities.
This facilitates new uses of conformal prediction, such as automatic decision making
\cite{\OCMXIX}.
However, for many underlying algorithms CPS (like conformal predictors in general)
are computationally inefficient:
CPS require re-training the underlying algorithm
for each test object and each postulated label for it,
and this can be done efficiently only for a narrow class of underlying algorithms,
including Least Squares \cite{\OCMXVII} and Kernel Ridge Regression \cite{\OCMXX}.
The main aim of this paper is to define and study computationally efficient versions of CPS
without restrictions on the underlying algorithm.

A very recent development in Venn prediction has been the introduction,
in the terminology of this paper,
of split Venn--Abers predictive systems in \cite{Nouretdinov/etal:2018COPA} (COPA 2018),
which are another way to produce predictive distributions.
A secondary aim of this paper is to explore several versions of Venn--Abers predictive systems
and compare them with conformal predictive systems.

We start, in Section~\ref{sec:RPS}, from defining randomized predictive systems (RPS).
In Section~\ref{sec:SCPS} we define their special case, split conformal predictive systems (SCPS),
which are computationally efficient but may suffer loss of predictive efficiency as compared with CPS
(which is indirectly confirmed in our experiments in Section~\ref{sec:experiments},
where SCPS typically suffer larger losses than their competitor that uses data more efficiently).
An important advantage of SCPS is that they are, similarly to CPS, provably valid;
in Section~\ref{sec:SCPS}, a suitable notion of validity is defined
and the validity of SCPS is demonstrated (by referring to a standard result).

In Section~\ref{sec:CCPS} we build cross-conformal predictive systems (CCPS)
on top of split conformal predictive systems.
In principle CCPS can lose their validity
(and therefore, formally are no longer RPS),
but in practice they usually satisfy the requirement of validity,
as defined in Section~\ref{sec:SCPS}
(cf.\ the experiments in \cite{\OCMVI} and Section~\ref{sec:experiments}).

Section~\ref{sec:experiments} is devoted to comparing the predictive efficiency of SCPS and CCPS
and exploring their empirical validity.
In this paper, we measure predictive efficiency of predictive distributions
using a loss function called continuous ranked probability score (CRPS).
This loss function and the way it is applied in our context
are defined in the preceding section, Section~\ref{sec:CRPS}.

Sections~\ref{sec:univ} and~\ref{sec:VA} discuss more general issues
(to be described momentarily).
Section~\ref{sec:conclusion} concludes and gives directions of further research.

The conference version of this paper was published in the Proceedings of COPA 2018
\cite{Vovk/etal:2018COPA},
and the journal version is to be published in \emph{Neurocomputing}.
As compared with the conference version,
in the journal version we added a more detailed comparison of SCPS and CCPS,
a detailed discussion of Venn--Abers predictive systems (in Section~\ref{sec:VA}),
and the analysis of universality of various predictive systems
(in Section~\ref{sec:univ}).
An important finding here is that SCPS and CCPS are universal,
whereas Venn--Abers predictive systems are not.

\section{Randomized predictive systems}
\label{sec:RPS}

Fix a nonempty measurable space $\mathbf{X}$;
we will refer to it as our \emph{object space}.
Define the \emph{observation space} as $\mathbf{Z}:=\mathbf{X}\times\R$;
each observation $z=(x,y)\in\mathbf{Z}$ consists of an object $x\in\mathbf{X}$ and its label $y\in\R$.

We will use the following definition, given in \cite{\OCMXVII}
(a modification of the definition in \cite[Definition~1]{Shen/etal:2017}).
Let $U$ be the uniform probability measure on the interval $[0,1]$.
\begin{definition}\label{def:RPS}
  A function $Q:\mathbf{Z}^{n+1}\times[0,1]\to[0,1]$ is a \emph{randomized predictive system} (RPS)
  if it satisfies the following three requirements:
  \begin{itemize}
  \item[R1]
    \begin{itemize}
    \item[i]
      For each training sequence $(z_1,\ldots,z_n)\in\mathbf{Z}^n$ and each test object $x\in\mathbf{X}$,
      the function $Q(z_1,\ldots,z_n,(x,y),\tau)$ is monotonically increasing both in $y$ and in $\tau$
      (where ``monotonically increasing'' is understood in the wide sense allowing intervals of constancy).
      In other words, for each $\tau\in[0,1]$,
      the function
      \begin{equation}\label{eq:function}
        y\in\R
        \mapsto
        Q(z_1,\ldots,z_n,(x,y),\tau)
      \end{equation}
      is monotonically increasing,
      and for each $y\in\R$,
      the function
      \[
        \tau\in[0,1]
        \mapsto
        Q(z_1,\ldots,z_n,(x,y),\tau)
      \]
      is also monotonically increasing.
    \item[ii]
      For each training sequence $(z_1,\ldots,z_n)\in\mathbf{Z}^n$ and each test object $x\in\mathbf{X}$,
      \begin{equation}\label{eq:y-minus-infty}
        \lim_{y\to-\infty}
        Q(z_1,\ldots,z_n,(x,y),0)
        =
        0
      \end{equation}
      and
      \begin{equation}\label{eq:y-infty}
        \lim_{y\to\infty}
        Q(z_1,\ldots,z_n,(x,y),1)
        =
        1.
      \end{equation}
    \end{itemize}
  \item[R2]
    For any probability measure $P$ on $\mathbf{Z}$,
    as function of random training observations $z_1\sim P$,\ldots, $z_n\sim P$,
    a random test observation $z\sim P$,
    and a random number $\tau\sim U$,
    all assumed independent,
    the distribution of $Q$ is uniform:
    \begin{equation}\label{eq:R2}
      \forall \alpha\in[0,1]:
      \Prob
      \left(
        Q(z_1,\ldots,z_n,z,\tau)
        \le
        \alpha
      \right)
      =
      \alpha.
    \end{equation}
  \end{itemize}
\end{definition}
The output
\begin{equation}\label{eq:Q}
  y\in\R\mapsto Q(z_1,\ldots,z_n,(x,y),\tau)
\end{equation}
of an RPS on a given training sequence $z_1,\ldots,z_n$,
test object $x$, and random number $\tau$
will be referred to as a \emph{predictive distribution (function)}.

\section{Split conformal predictive systems}
\label{sec:SCPS}

In this section we will modify the definitions of conformal predictive systems
given in \cite{\OCMXVII} along the lines of \cite[Section 2.3]{Bala/etal:2014}
(removing an unnecessary assumption in \cite[Section 4.1]{Vovk/etal:2005book}).
A \emph{split conformity measure}
is a family of measurable functions $A_m:\mathbf{Z}^{m+1}\to\R\cup\{-\infty,\infty\}$, $m=1,2,\ldots$\,.
The intention is that $A_m(z_1,\ldots,z_{m+1})$ measures how large the label $y_{m+1}$ in $z_{m+1}$ is,
as compared with the labels in $z_1,\ldots,z_m$.
Suppose the training sequence $z_1,\ldots,z_n$ is split into two parts:
the \emph{training sequence proper} $z_1,\ldots,z_m$ and the \emph{calibration sequence} $z_{m+1},\ldots,z_n$;
we are given a test object $x$.
The output of the \emph{split conformal transducer} determined by the split conformity measure $A$ is defined as
\begin{multline}\label{eq:SCPS}
  Q(z_1,\ldots,z_n,(x,y),\tau)
  :=
  \frac{1}{n-m+1}
  \left|\left\{i=m+1,\ldots,n\mid\alpha_i<\alpha^y\right\}\right|\\
  +
  \frac{\tau}{n-m+1}
  \left|\left\{i=m+1,\ldots,n\mid\alpha_i=\alpha^y\right\}\right|
  +
  \frac{\tau}{n-m+1},
\end{multline}
where the \emph{conformity scores} $\alpha_i$, $i=m+1,\ldots,n$, and $\alpha^y$, $y\in\R$,
are defined by
\begin{equation*}
  \begin{aligned}
    \alpha_i
    &:=
    A(z_1,\ldots,z_m,(x_i,y_i)),
      \qquad i=m+1,\ldots,n,\\
    \alpha^y
    &:=
    A(z_1,\ldots,z_m,(x,y)).
  \end{aligned}
\end{equation*}
(We omit the lower index $m$ in $A_m$ since it is determined by the number of arguments.)
A function is a \emph{split conformal transducer}
if it is the split conformal transducer determined by some split conformity measure.
A \emph{split conformal predictive system} (SCPS) is a function which is both a split conformal transducer
and a randomized predictive system.

The standard property of validity (satisfied automatically) for split conformal transducers
is that the values $Q(z_1,\ldots,z_n,z,\tau)$ are distributed uniformly on $[0,1]$
when $z_1,\ldots,z_n,z$ are IID and $\tau$ is generated independently of $z_1,\ldots,z_n,z$
from the uniform probability distribution $U$ on $[0,1]$
(see, e.g., \cite[Proposition~4.1]{Vovk/etal:2005book}).

It is much easier to get an RPS using split conformal transducers than using conformal transducers.
A split conformity measure $A$ is \emph{isotonic} if, for all $m$, $z_1,\ldots,z_m$, and $x$,
$A(z_1,\ldots,z_m,(x,y))$ is isotonic in $y$, i.e.,
\begin{equation}\label{eq:isotonic}
  y \le y'
  \Longrightarrow
  A(z_1,\ldots,z_m,(x,y))
  \le
  A(z_1,\ldots,z_m,(x,y'))
\end{equation}
(cf.\ \cite{\OCMXVII}, the definition of monotonic conformity measures in Section~2).
An isotonic split conformity measure $A$ is \emph{balanced} if,
for any $m$ and $z_1,\ldots,z_m$,
the set
\begin{equation}\label{eq:image}
  \conv A(z_1,\ldots,z_m,(x,\R))
  :=
  \conv
  \left\{
    A(z_1,\ldots,z_m,(x,y))
    \mid
    y\in\R
  \right\}
\end{equation}
does not depend on $x$,
where $\conv$ stands for the convex closure in $\R$.
The set~\eqref{eq:image} then coincides with $\conv A(z_1,\ldots,z_m,\mathbf{Z})$
and has one of four forms:
$(a,b)$, $[a,b)$, $(a,b]$, or $[a,b]$,
where $a<b$ are elements of the extended real line $\R\cup\{-\infty,\infty\}$;
in this paper, we will be mainly interested in the case
$\conv A(z_1,\ldots,z_m,\mathbf{Z})=(-\infty,\infty)$.

\begin{proposition}
  The split conformal transducer \eqref{eq:SCPS}
  based on a balanced isotonic split conformity measure
  is an RPS.
\end{proposition}

\begin{proof}
  Since property R2 is automatic, we only need to check R1.
  It is clear that \eqref{eq:SCPS} is increasing in $\tau$ (and linear).

  To show that it is increasing in $y$, split, in the context of \eqref{eq:SCPS}, all $i\in\{m+1,\ldots,n\}$
  into three groups:
  the $i$ in group 1 satisfy $\alpha_i<\alpha^y$,
  the $i$ in group 2 satisfy $\alpha_i=\alpha^y$,
  and the $i$ in group 3 satisfy $\alpha_i>\alpha^y$.
  Then \eqref{eq:SCPS} is the total weight of all $i$ where the weights are $1$, $\tau\in[0,1]$, and $0$
  for $i$ in groups 1, 2, and 3, respectively.
  As $y$ increases, $\alpha^y$ increases as well,
  and therefore, each $i$ can only move to a lower-numbered group thus increasing \eqref{eq:SCPS}.

  Out of the remaining two conditions, let us check, e.g., \eqref{eq:y-infty}.
  It suffices to notice that, since $A$ is balanced,
  we have $\alpha^y\ge\max_{i\in\{m+1,\ldots,n\}}\alpha_i$ from some $y$ on, for any $z_1,\ldots,z_n$ and~$x$.
\end{proof}

The next proposition shows that a split conformity measure being isotonic and balanced
is not only a sufficient but also a necessary condition for the corresponding split conformal transducer to be an RPS.

\begin{proposition}
  If the split conformal transducer based on a split conformity measure $A$ is an RPS,
  $A$ is isotonic and balanced.
\end{proposition}

\begin{proof}
  Suppose $A$ is not isotonic.
  Fix $m$, $z_1,\ldots,z_m$, $x$, $y$, and $y'$ such that $y<y'$ but the consequent of \eqref{eq:isotonic} is violated.
  Then the putative predictive distribution $Q(z_1,\ldots,z_m,(x,y),(x,\cdot),1)$,
  corresponding to the training sequence proper $z_1,\ldots,z_m$, calibration sequence $(x,y)$, test object $x$, and $\tau=1$,
  will not be increasing:
  its value at $y$ (which is $1$) will be greater than its value at $y'$ (which is $0.5$).

  Now suppose $A$ is not balanced.
  Fix $m$, $z_1,\ldots,z_m$, and $x,x'\in\mathbf{X}$ such that
  \[
    \conv A(z_1,\ldots,z_m,(x,\R))
    \ne
    \conv A(z_1,\ldots,z_m,(x',\R))
  \]
  (cf.\ \eqref{eq:image}).
  Suppose, for concreteness, that there is $y\in\R$ such that
  \[
    \conv A(z_1,\ldots,z_m,(x,\R))
    \ni
    y
    <
    \conv A(z_1,\ldots,z_m,(x',\R)),
  \]
  where $y<S$ means $\forall s\in S:y<s$ when $S\subseteq\R$.
  (The other three possible cases can be analyzed in the same way.)
  Let the training sequence proper be $z_1,\ldots,z_m$,
  the calibration sequence be $(x,y)$,
  the test object be $x'$,
  and the random number be $\tau=0$.
  Then we will have
  \begin{equation*}
    \lim_{y'\to-\infty}
    Q(z_1,\ldots,z_m,(x,y),(x',y'),0)
    >
    0,
  \end{equation*}
  which contradicts R1 (cf.\ \eqref{eq:y-minus-infty}).
\end{proof}

\begin{algorithm}[bt]
  \caption{Split Conformal Predictive System}
  \label{alg:SCPS}
  \begin{algorithmic}
    \Require
      A training sequence $(x_i,y_i)\in\mathbf{Z}$, $i=1,\ldots,n$.
    \Require
      A test object $x\in\mathbf{X}$.
    \For{$i\in\{1,\ldots,n-m\}$}
      \State Define $C_i$ by the condition $A(z_1,\ldots,z_m,z_{m+i})=A(z_1,\ldots,z_m,(x,C_i))$.
    \EndFor
    \State Sort $C_1,\ldots,C_{n-m}$ in the increasing order obtaining $C_{(1)}\le\cdots\le C_{(n-m)}$.
    \State Set $C_{(0)}:=-\infty$ and $C_{(n-m+1)}:=\infty$.
    \State Return the predictive distribution \eqref{eq:Q-1} for the label $y$ of $x$.
  \end{algorithmic}
\end{algorithm}

Let us say that a split conformity measure $A$ is \emph{strictly isotonic}
if \eqref{eq:isotonic} holds with both ``$\le$'' replaced by ``$<$''.
A possible implementation of the SCPS based on a balanced strictly isotonic split conformity measure is shown as Algorithm~\ref{alg:SCPS},
where the predictive distribution is defined by
\begin{multline}\label{eq:Q-1}
  Q(z_1,\ldots,z_n,(x,y),\tau)
  :={}\\
  \begin{cases}
    \frac{i+\tau}{n-m+1} & \text{if $y\in(C_{(i)},C_{(i+1)})$ for $i\in\{0,1,\ldots,n-m\}$}\\[2mm]
    \frac{i'-1+(i''-i'+2)\tau}{n-m+1} & \text{if $y=C_{(i)}$ for $i\in\{1,\ldots,n-m\}$},
  \end{cases}
\end{multline}
where $i':=\min\{j\mid C_{(j)}=C_{(i)}\}$ and $i'':=\max\{j\mid C_{(j)}=C_{(i)}\}$.
To use the terminology of \cite{\OCMXVII},
the thickness of this predictive distribution is $\frac{1}{n-m+1}$
with the exception size at most $n-m$.

How computationally efficient Algorithm~\ref{alg:SCPS} is depends
on how easy to solve the equation defining $C_i$ is.
A standard choice of split conformity measure is
\begin{equation}\label{eq:example}
  A(z_1,\ldots,z_m,(x,y))
  :=
  \frac{y-\hat y}{\hat\sigma},
\end{equation}
where $\hat y$ is a prediction for the label $y$
computed from $x$ as test object and $z_1,\ldots,z_m$ as training sequence,
and $\hat\sigma$ is an estimate of the quality of $\hat y$ computed from the same data.
In this case the equation
\begin{equation}\label{eq:C_i}
  A(z_1,\ldots,z_m,z_{m+i})
  =
  A(z_1,\ldots,z_m,(x,C_i))
\end{equation}
defining $C_i$ becomes
\[
  \frac{y_{m+i} - \hat y_{m+i}}{\hat\sigma_{m+i}}
  =
  \frac{C_i - \hat y}{\hat\sigma},
\]
where $\hat y_{m+i}$ (resp.\ $\hat y$) is the prediction for $y_{m+i}$ (resp.\ $y$)
computed from $x_{m+i}$ (resp.\ $x$) as test object and $z_1,\ldots,z_m$ as training sequence,
and $\hat\sigma_{m+i}$ (resp.\ $\hat\sigma$) is the estimate of the quality of $\hat y_{m+i}$ (resp.\ $\hat y$)
computed from the same data.
The last equation allows us to set
\[
  C_i
  :=
  \hat y
  +
  \frac{\hat\sigma}{\hat\sigma_{m+i}}
  \left(
    y_{m+i} - \hat y_{m+i}
  \right).
\]
For more complicated split conformity measures $A$,
it might be more efficient to use the expression~\eqref{eq:SCPS} directly
for a grid of values of $y$.

\section{Cross-conformal predictive distributions}
\label{sec:CCPS}

Remember that a \emph{multiset} (or bag) is different from a set in that it can contain several copies of the same element.
A split conformity measure $A$ is a \emph{cross-conformity  measure}
if $A(z_1,\ldots,z_m,z)$ does not depend on the order of its first $m$ arguments;
in other words, if $A(z_1,\ldots,z_m,z)$ only depends on the multiset $\lbag z_1,\ldots,z_m\rbag$ and $z$
(where $\lbag\cdots\rbag$ is used as the analogue of $\{\cdots\}$ for multisets).

Given a balanced isotonic cross-conformity measure $A$,
the corresponding \emph{cross-conformal predictive system} (CCPS) is defined as follows.
The training sequence $z_1,\ldots,z_n$ is randomly split into $K$ non-empty multisets (\emph{folds})
$z_{S_k}$, $k=1,\ldots,K$, of equal (or as equal as possible) sizes,
where $K\in\{2,3,\ldots\}$ is a parameter of the algorithm,
$(S_1,\ldots,S_K)$ is a partition of the index set $\{1,\ldots,n\}$,
and $z_{S_k}$ consists of all $z_i$, $i\in S_k$.
For each $k\in\{1,\ldots,K\}$ and each potential label $y\in\R$ of the test object $x$,
find the conformity scores of the observations in $z_{S_k}$ and of $(x,y)$ by
\begin{equation*}
  \alpha_{i,k} := A(z_{S_{-k}},z_i),
  \quad
  i\in S_k,
  \qquad
  \alpha^y_k := A(z_{S_{-k}},(x,y)),
\end{equation*}
where $S_{-k}:=\cup_{j\ne k}S_j=\{1,\ldots,n\}\setminus S_k$.
The corresponding p-values and CCPS are defined by
\begin{multline}\label{eq:CCPS}
  p^y
  =
  Q(z_1,\ldots,z_n,(x,y),\tau)
  :=
  \frac{1}{n+1}
  \sum_{k=1}^K\left|\left\{i\in S_k\mid\alpha_{i,k}<\alpha^y_k\right\}\right|\\
  +
  \frac{\tau}{n+1}
  \sum_{k=1}^K\left|\left\{i\in S_k\mid\alpha_{i,k}=\alpha^y_k\right\}\right|
  +
  \frac{\tau}{n+1}.
\end{multline}
The intuition behind \eqref{eq:CCPS} is that it becomes an SCPS when the training multisets $z_{S_{-k}}$
are replaced by a single hold-out training sequence (one disjoint from and independent of $z_1,\ldots,z_n$).

\begin{algorithm}[bt]
  \caption{Cross-Conformal Predictive System}
  \label{alg:CCPS}
  \begin{algorithmic}
    \Require
      A training sequence $(x_i,y_i)\in\mathbf{Z}$, $i=1,\ldots,n$.
    \Require
      A test object $x\in\mathbf{X}$.
    \State Split $z_1,\ldots,z_n$ into $K$ folds $z_{S_k}$ as described in text.
    \State Set $C:=\emptyset$, where $C$ is a multiset.
    \For{$k\in\{1,\ldots,K\}$}
      \For{$i\in S_k$}
        \State Define $C_{i,k}$ by the condition $A(z_{S_{-k}},z_{i})=A(z_{S_{-k}},(x,C_{i,k}))$.
        \State Put $C_{i,k}$ in $C$.
      \EndFor
    \EndFor
    \State Sort $C$ in the increasing order obtaining $C_{(1)}\le\cdots\le C_{(n)}$.
    \State Set $C_{(0)}:=-\infty$ and $C_{(n+1)}:=\infty$.
    \State Return the predictive distribution \eqref{eq:Q-2} for the label $y$ of $x$.
  \end{algorithmic}
\end{algorithm}

An implementation of the CCPS based on a balanced strictly isotonic cross-conformity measure
is shown as Algorithm~\ref{alg:CCPS},
where the predictive distribution is now defined by
\begin{multline}\label{eq:Q-2}
  Q(z_1,\ldots,z_n,(x,y),\tau)
  := {}\\
  \begin{cases}
    \frac{i+\tau}{n+1} & \text{if $y\in(C_{(i)},C_{(i+1)})$ for $i\in\{0,1,\ldots,n\}$}\\[2mm]
    \frac{i'-1+(i''-i'+2)\tau}{n+1} & \text{if $y=C_{(i)}$ for $i\in\{1,\ldots,n\}$},
  \end{cases}
\end{multline}
where, as before, $i':=\min\{j\mid C_{(j)}=C_{(i)}\}$ and $i'':=\max\{j\mid C_{(j)}=C_{(i)}\}$;
the only difference from \eqref{eq:Q-1} is that we use $n$ in place of $n-m$
(now all training observations are used for calibration).
The thickness of this predictive distribution is $\frac{1}{n+1}$ with the exception size at most $n$.
The size of the multiset $C$ in Algorithm~\ref{alg:CCPS} grows from $0$ to $n$ as the algorithm runs.
As in the case of SCPS, it might be easier to use \eqref{eq:CCPS} directly
if the equations defining $C_{i,k}$ are difficult to solve.
(Alternatively, one could use \eqref{eq:modified-mean} below instead of \eqref{eq:CCPS}.)

Define a separate p-value
\begin{multline}\label{eq:CCPS-separate}
  p^y_k
  :=
  \frac{1}{\left|S_k\right|+1}
  \left|\left\{i\in S_k\mid\alpha_{i,k}<\alpha^y_k\right\}\right| \\
  +
  \frac{\tau}{\left|S_k\right|+1}
  \left|\left\{i\in S_k\mid\alpha_{i,k}=\alpha^y_k\right\}\right|
  +
  \frac{\tau}{\left|S_k\right|+1}
\end{multline}
for each fold
(cf.\ \eqref{eq:SCPS});
let us check that $p^y$ is close to being an average of $p^y_k$.
Comparing~\eqref{eq:CCPS} and~\eqref{eq:CCPS-separate}, we can see that
\[
  (n+1) p^y
  -
  \tau
  =
  \sum_{k=1}^K
  \left(
    \left|S_k\right|+1
  \right)
  p^y_k
  -
  K\tau,
\]
which implies
\begin{equation}\label{eq:modified-mean}
  p^y
  =
  \sum_{k=1}^K
  \frac{\left|S_k\right|+1}{n+1}
  p^y_k
  -
  \frac{K-1}{n+1}\tau.
\end{equation}
The sum $\sum_{k=1}^K\dots$ is not quite a weighted average of $p^y_k$
since the sum of the weights is slightly above 1
(``slightly'' assumes $K\ll n$),
but this is partially compensated by the subtrahend in \eqref{eq:modified-mean};
overall, the right-hand side of~\eqref{eq:modified-mean} is a weighted average of $p^y_k$ and $\tau$,
with the weight in front of $\tau$ being negative.

According to the intuition behind cross-conformal predictive distributions described earlier,
we will get perfect validity for CCPS
if we replace the $K$ training multisets (the complements to the $K$ folds) by one hold-out training sequence.
But whereas SCPS are provably valid, in the sense of being RPS,
real CCPS are not RPS: see the example in \cite[Appendix A]{\OCMVI}.
In experimental studies,
this phenomenon has been demonstrated by \cite{Linusson/etal:2017},
who showed the danger of randomized and extremely unstable underlying algorithms.
(Perhaps such unstable algorithms might be stabilized, to some degree,
by using the same seed of the random numbers generator for each fold,
or by averaging conformity scores over several seeds, or both.)
A useful intuition \cite{Linusson/etal:2017} is that the random p-values coming from different folds
(and then essentially averaged by cross-conformal predictors)
are to some degree independent,
and so the distribution of cross-conformal p-values is intermediate between the uniform and the Bates distributions;
therefore, cross-conformal p-values are conservative when not exact
(for small significance levels).
According to a result in \cite{\OCMXXI} (see, e.g., Table~1 for $r:=1$),
we will get provably valid (but perhaps conservative) p-values
if we multiply the p-values output by a cross-conformal transducer by 2;
the empirical fact observed by \cite{Linusson/etal:2017}
is that for randomized and unstable underlying algorithms
even unadjusted p-values output by a cross-conformal transducer are valid but perhaps overly conservative
for interesting (not exceeding $0.5$) significance levels.

A more general procedure than the cross-conformal predictor was proposed in \cite{Carlsson/etal:2014}
under the name of ``aggregated conformal predictor''.
Similar methods might be applicable for producing conformal predictive distributions.

\section{Continuous ranked probability score}
\label{sec:CRPS}

Suppose the prediction for a label $y\in\R$ is a distribution function $F:\R\to[0,1]$ and the observed value of $y$ is $y_i$.
The quality of the prediction $F$ in view of the actual outcome $y_i$ is often measured
by the \emph{continuous ranked probability score}
\begin{equation}\label{eq:CRPS}
  \CRPS(F,y_i)
  :=
  \int_{-\infty}^{\infty}
  \left(
    F(y)
    -
    \II_{\{y\ge y_i\}}
  \right)^2
  \dd y,
\end{equation}
where $\II$ stands for the indicator function.
The lowest possible value 0 is attained when $F$ is concentrated at $y_i$,
and in all other cases $\CRPS(F,y_i)$ will be positive.
(See, e.g., \cite{Gneiting/Katzfuss:2014} for further details and references.)

\begin{figure}
  \begin{center}
    \includegraphics[width=0.8\textwidth]{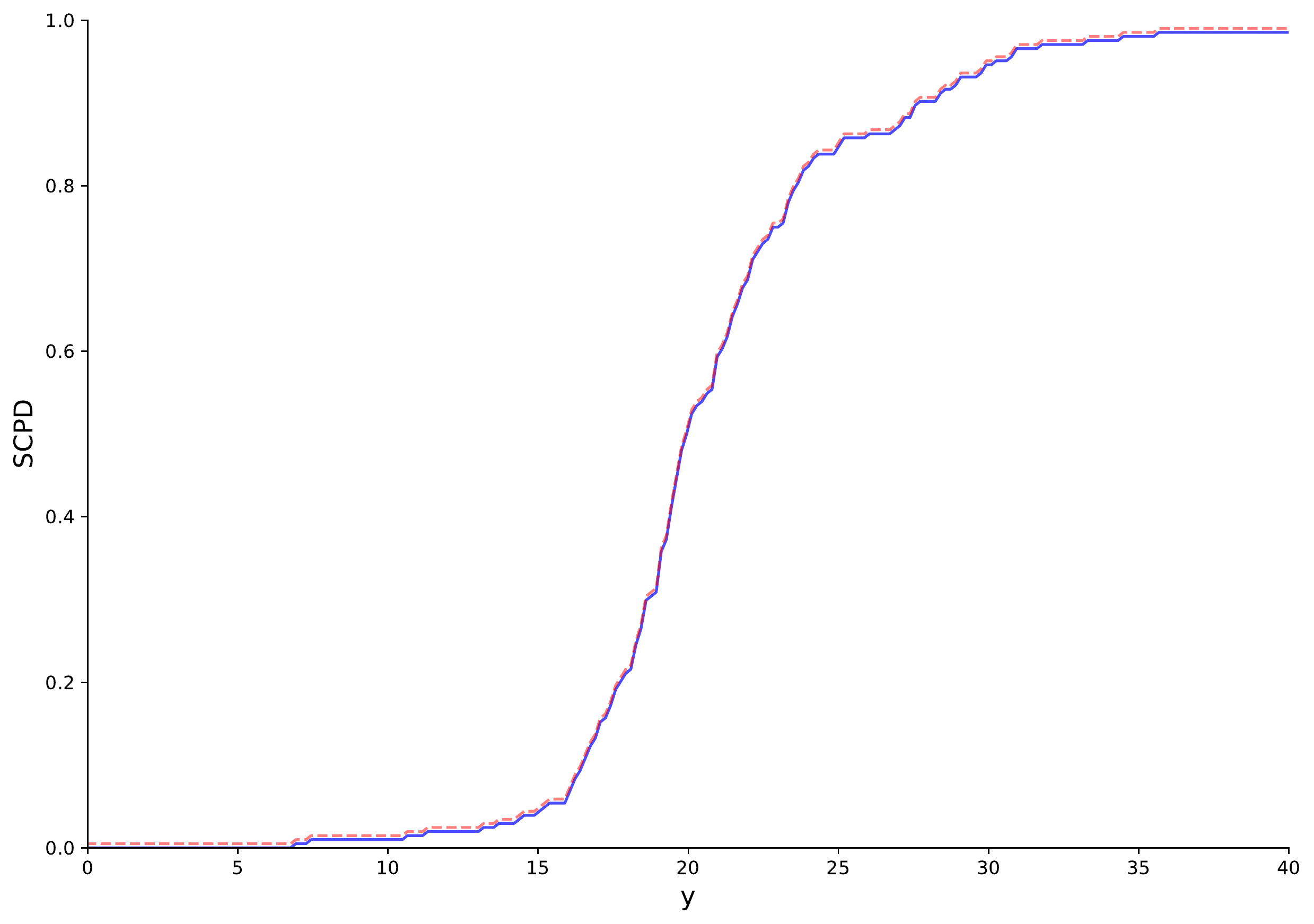}
  \end{center}
  \caption{The split conformal predictive distribution for a random test object in the \texttt{Boston Housing} dataset
    (described in Section~\ref{sec:experiments}),
    the Least Squares underlying algorithm, and a random $50\%:50\%$ split of the training sequence
    into proper training and calibration sequences.
    The blue solid line corresponds to $\tau = 0$ and the red dashed line to $\tau = 1$.}
  \label{fig:boston_random}
\end{figure}

Strictly speaking, \eqref{eq:CRPS} is not applicable
to split and cross-conformal predictive distributions, which are somewhat ``fuzzy''
(the thickness for the former is $\frac{1}{n-m+1}$ and for the latter it is $\frac{1}{n+1}$).
In practice, the fuzziness can usually be ignored,
even for relatively small datasets: see, e.g., Figure~\ref{fig:boston_random}.
However, conceptually we do need to change split and cross-conformal predictive distributions slightly
to remove their fuzziness.

Instead of \eqref{eq:Q-1} and \eqref{eq:Q-2}
we use their crisp modifications
\begin{equation}\label{eq:Q-1-crisp}
  Q(z_1,\ldots,z_n,(x,y))
  :=
  \begin{cases}
    \frac{i}{n-m} & \text{if $y\in(C_{(i)},C_{(i+1)})$ for $i\in\{0,1,\ldots,n-m\}$}\\[1mm]
    \frac{i}{n-m} & \text{if $y=C_{(i)}$ and $y\ne C_{(i+1)}$ for $i\in\{1,\ldots,n-m\}$}
  \end{cases}
\end{equation}
and
\begin{equation}\label{eq:Q-2-crisp}
  Q(z_1,\ldots,z_n,(x,y))
  :=
  \begin{cases}
    \frac{i}{n} & \text{if $y\in(C_{(i)},C_{(i+1)})$ for $i\in\{0,1,\ldots,n\}$}\\[1mm]
    \frac{i}{n} & \text{if $y=C_{(i)}$ and $y\ne C_{(i+1)}$ for $i\in\{1,\ldots,n\}$},
  \end{cases}
\end{equation}
respectively;
these modifications no longer depend on $\tau$,
and the convention for $y=C_{(i)}$ does not affect the value of CRPS.
In cases where the equation~\eqref{eq:C_i} or its analogue for the CCPS are difficult to solve,
we can instead use the following crisp modifications of~\eqref{eq:SCPS} and~\eqref{eq:CCPS}, respectively:
\begin{align*}
  Q(z_1,\ldots,z_n,(x,y))
  &:=
  \frac{1}{n-m}
  \left|\left\{i=m+1,\ldots,n\mid\alpha_i\le\alpha^y\right\}\right|,\\
  Q(z_1,\ldots,z_n,(x,y))
  &:=
  \frac{1}{n}
  \sum_{k=1}^K\left|\left\{i\in S_k\mid\alpha_{i,k}\le\alpha^y_k\right\}\right|.
\end{align*}
The last equation, defining a crisp CCPS,
can be rewritten as
\begin{equation*}
  Q(z_1,\ldots,z_n,(x,y))
  =
  \sum_{k=1}^K
  \frac{\left|S_k\right|}{n}
  p^y_k
\end{equation*}
(cf.\ \eqref{eq:modified-mean}),
where the separate ``p-values'' for each fold are now defined as
\begin{equation*}
  p^y_k
  :=
  \frac{1}{\left|S_k\right|}
  \left|\left\{i\in S_k\mid\alpha_{i,k}\le\alpha^y_k\right\}\right|
\end{equation*}
(they, however, do not satisfy any validity properties).

\section{Experiments}
\label{sec:experiments}

The purpose of this section is to compare the predictive performance of SCPS and CCPS
and to recommend the choice of the parameter $K$ for CCPS.
Our choice of conformity measures might well be improved in future work
(cf.\ Section~\ref{sec:univ}).

In our experiments we use five well-known benchmark datasets,
namely \texttt{Boston Housing}, \texttt{Diabetes}, \texttt{Yacht Hydrodynamics},
\texttt{Wine Quality}, and \texttt{Condition Based Maintenance of Naval Propulsion Plants}
(abbreviated to \texttt{Naval Propulsion})
available at
\href{http://scikit-learn.org/stable/datasets/}{http://scikit-learn.org/stable/datasets/}
(the first two)
and the UCI Machine Learning repository \cite{UCI:2017}
(the other ones).
The first three datasets are small:
\texttt{Boston Housing} consists of 506 observations,
\texttt{Diabetes} of 442 observations,
and \texttt{Yacht Hydrodynamics} of 308 observations;
for them we use test sequences of length $l:=100$.
The \texttt{Wine Quality} dataset consists of 6497 observations,
and we use test sequences of length $l:=1000$.
Finally, the \texttt{Naval Propulsion} dataset consists of 11,934 observations,
and we use test sequences of length $l:=4000$.

\begin{algorithm}[bt]
  \caption{Experiments in Section~\ref{sec:experiments}}
  \label{alg:experiments}
  \begin{algorithmic}[1]
    \Require
      A dataset of size $n+l$ consisting of observations $z=(x,y)$, $y\in\R$.
    \For{$U\in\{\text{LS},\text{RF},\text{NN}\}$}\label{ln:U-1}
      \For{$\alpha\in\{0.01,0.05,\dots,0.99\}$}
        \State Create an empty multiset $B_{U,\alpha}$.
      \EndFor
    \EndFor
    \For{$U\in\{\text{LS},\text{RF},\text{NN}\}$}\label{ln:U-2}
      \For{$K\in\{2,3,\dots,100\}$}
        \State Create an empty multiset $B'_{U,K}$.
      \EndFor
    \EndFor
    \For{10 times}
      \State Randomly permute the dataset obtaining a sequence $z_1,\dots,z_{n+l}$.
      \State Use $z_1,\dots,z_n$ as the training and $z_{n+1},\dots,z_{n+l}$ as the test sequence.
      \State Apply feature scaling by fitting on the training sequence and\AI transforming the training and test sequences.
      \State Tune the parameters using 3-fold cross-validation on the training\label{ln:tune-1}\AI sequence.\label{ln:tune-2}
      \For{$\alpha\in\{0.01,0.05,\dots,0.99\}$}
        \State $m:=\lfloor\alpha n\rfloor$.
        \For{$U\in\{\text{LS},\text{RF},\text{NN}\}$}\label{ln:U-3}
          \State Train an SCPS based on $U$ using $z_1,\dots,z_m$ as training sequence\AI proper and $z_{m+1},\dots,z_n$ as calibration sequence.
          \State Put all $\CRPS(F_i,y_i)$, $i\in\{n+1,\ldots,n+l\}$ in $B_{U,\alpha}$, where $F_i$ is\AI the output of the SCPS for $x_i$.
        \EndFor
      \EndFor
      \For{$K\in\{2,3,\dots,100\}$}
        \State Put all $z_i$, $i\in\{\lceil(k-1)n/K\rceil+1,\lceil k n/K\rceil\}$, into fold $k$, $k\in\{1,\dots,K\}$.
        \For{$U\in\{\text{LS},\text{RF},\text{NN}\}$}\label{ln:U-4}
          \State Train a CCPS based on $U$ using these folds.
          \State Put all $\CRPS(F_i,y_i)$, $i\in\{n+1,\ldots,n+l\}$ in $B'_{U,K}$, where $F_i$ is\AI the output of the CCPS for $x_i$.
        \EndFor
      \EndFor
    \EndFor
    \For{$U\in\{\text{LS},\text{RF},\text{NN}\}$}\label{ln:U-5}
      \For{$\alpha\in\{0.01,0.05,\dots,0.99\}$}
        \State Show the multiset $B_{U,\alpha}$ as boxplot.
      \EndFor
    \EndFor
    \For{$U\in\{\text{LS},\text{RF},\text{NN}\}$}\label{ln:U-6}
      \For{$K\in\{2,3,\dots,100\}$}
        \State Show the multiset $B'_{U,K}$ as boxplot.
      \EndFor
    \EndFor
  \end{algorithmic}
\end{algorithm}

Given a training sequence $(z_1,\ldots,z_n)$
(where $n\in\{406,342,208,5497,7934\}$)
and a test sequence $(z_{n+1},\ldots,z_{n+l})$,
the quality of prediction is represented
by the distribution of $\CRPS(F_i,y_i)$, $i=n+1,\ldots,n+l$,
where $F_i$ is the predictive distribution for the label $y_i$ of the test object $x_i$.
As already mentioned, the length $l$ of the test sequence is 100, 1000, or 4000 in our experiments.

In order to obtain boxplots less affected by the split of each dataset
into a training and test sequence
and by the random split of each training sequence
into a training sequence proper and a calibration sequence (in the case of SCPS)
or $K$ folds (in the case of CCPS),
we use the procedure given as Algorithm~\ref{alg:experiments}.
Each dataset is randomly permuted 10 times.
The last $l$ observations of each permutation are used for testing and the rest for training.
The first $m$ observations in the training sequence are used as training sequence proper
in the case of SCPS
and consecutive blocks of the training sequence are used as the $K$ folds
in the case of CCPS
(using the \texttt{scikit-learn} \texttt{KFold} procedure with no randomization).
The boxplots in all figures given below
are indexed by the fractions $m/n$ of the training sequence
used as the training sequence proper (in the case of SCPS)
or by the numbers $K$ of folds (in the case of CCPS).
For each split and each boxplot we find the $l$ values $\CRPS(F_i,y_i)$
for all test observations (the same test sequence is used for each split);
the resulting boxplot is based on all $10\,l$ numbers.

The loops in lines \ref{ln:U-1}, \ref{ln:U-2}, \ref{ln:U-3}, \ref{ln:U-4}, \ref{ln:U-5}, and \ref{ln:U-6}
of Algorithm~\ref{alg:experiments} are over our underlying algorithms $U$
(Least Squares, Random Forest, and Neural Networks,
as implemented in \texttt{scikit-learn}).
In all cases the SCPS and CCPS use the cross-conformity measure
(a special case of \eqref{eq:example})
\begin{equation}\label{eq:example-special}
  A(z_1,\ldots,z_m,(x,y))
  :=
  y-\hat y,
\end{equation}
where $\hat y$ is the prediction computed using the underlying algorithm $U$
for the label of $x$
based on $z_1,\ldots,z_m$ as training sequence.
(Remember that each cross-conformity measure is also a split conformity measure.)
Similarly to the CPS based on Least Squares \cite{\OCMXVII} and Kernel Ridge Regression \cite{\OCMXX}
(as discussed above),
this procedure is far from universal and can be expected to be efficient only for data
that is not too far from being homoscedastic;
this will be further discussed in Section~\ref{sec:univ}
(see, in particular, Proposition~\ref{prop:simple}).

Notice that, when implemented as in Algorithm~\ref{alg:experiments},
the SCPS is no longer provably calibrated
(because parameter tuning in lines~\ref{ln:tune-1}--\ref{ln:tune-2}
depends on the full training sequence),
and this is why we also check its validity in our experiments.
To check the validity of both SCPS and CCPS,
we run Algorithm~\ref{alg:experiments} replacing $\CRPS(F_i,y_i)$ with $F_i(y_i)$
and replacing boxplots with plots,
such as those in Figure~\ref{fig:calibration_boston}
(described in detail at the end of this section).

\begin{figure}
  \begin{center}
    \includegraphics[width=0.99\textwidth]{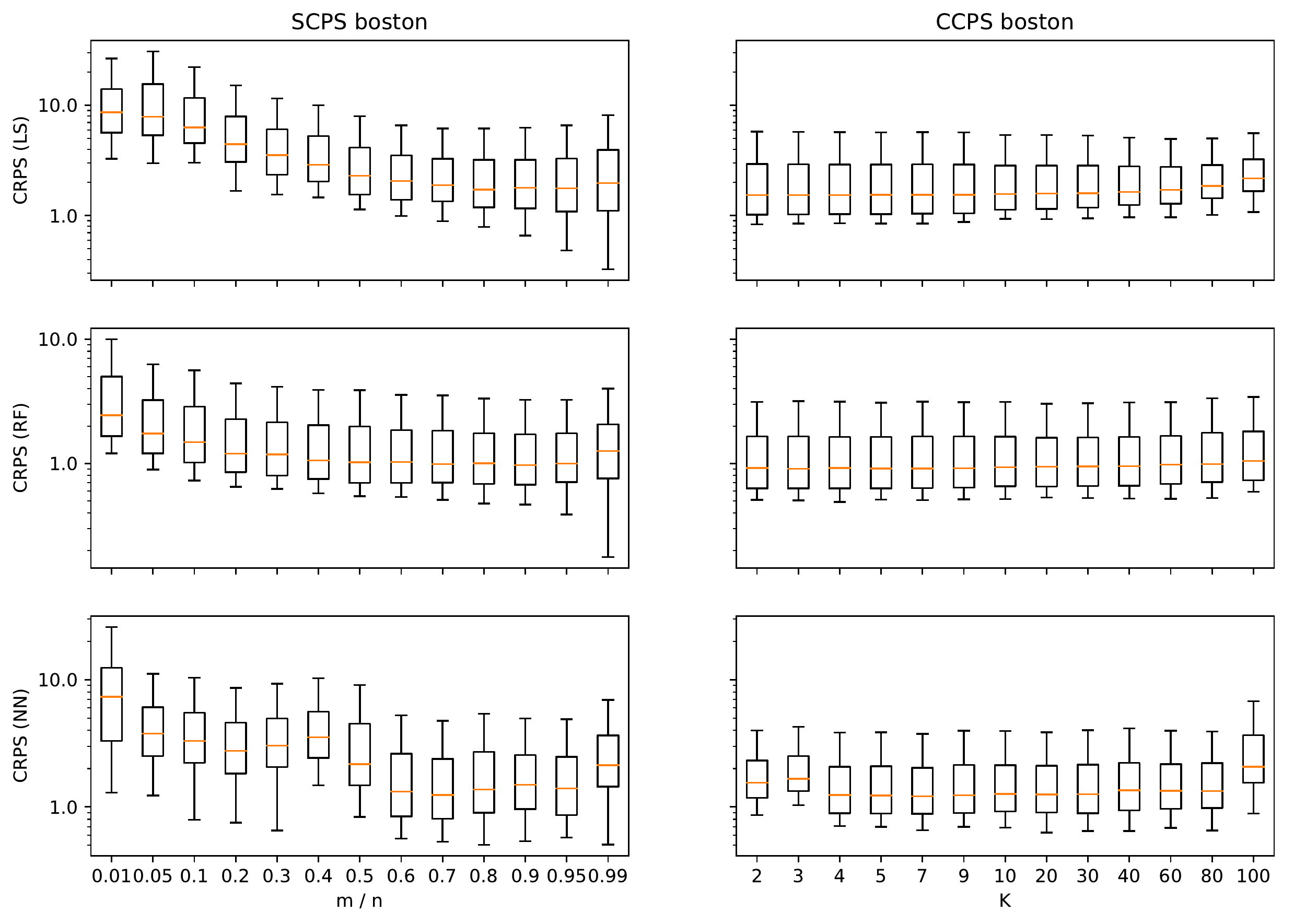}
  \end{center}
  \caption{The performance of the SCPS (left panel) and CCPS (right panel) on the \texttt{Boston Housing} dataset
    using Least Squares (LS), Random Forest (RF), and Neural Networks (NN) as the underlying algorithms,
    as indicated on the left.
    The vertical axis uses the log scale and gives the CRPS.
    Left panel: the numbers on the horizontal axis are the fractions $m/n$ of the training sequence
    used as the training sequence proper.
    Right panel: the numbers on the horizontal axis are the numbers $K$ of folds.}
  \label{fig:boston}
\end{figure}

The \texttt{Boston Housing} dataset consists of 506 observations
each with 14 attributes (describing an area of Boston)
and a real-valued label (median house price in that area).
Figure~\ref{fig:boston} shows the performance of the SCPS and CCPS.

The horizontal axis in the left panel is labelled by $\alpha\approx m/n$;
the values of $\alpha$ used in our experiments
are between $0.1$ and $0.9$, plus a few more extreme values.
For a given value of $\alpha$ we set $m:=\lfloor\alpha n\rfloor$.
The CRPS loss is computed for the (crisp) SCPS based on \eqref{eq:example-special}
and the three underlying algorithms
on each observation in the test sequence;
as described above,
we then represent the resulting 1000 CRPS losses as a boxplot.
We can see a characteristic U-shape (especially pronounced on the left);
small $m/n$ lead to a significant increase in the CRPS loss,
and large $m/n$ lead to a slight increase in the CRPS loss
but a significant increase in its variability
(the rightmost box and its whiskers tend to be longer).

The right panel of Figure~\ref{fig:boston} is similar to the left panel,
but now we use the CCPS and label the horizontal axis by the number $K$ of folds.
The usual advice in cross validation is to use $K\in\{5,10\}$,
and these two values produce reasonable results.
In fact, the results are remarkably stable and barely depend on $K$.

\begin{figure}[bt]
  \begin{center}
    \includegraphics[width=0.99\textwidth]{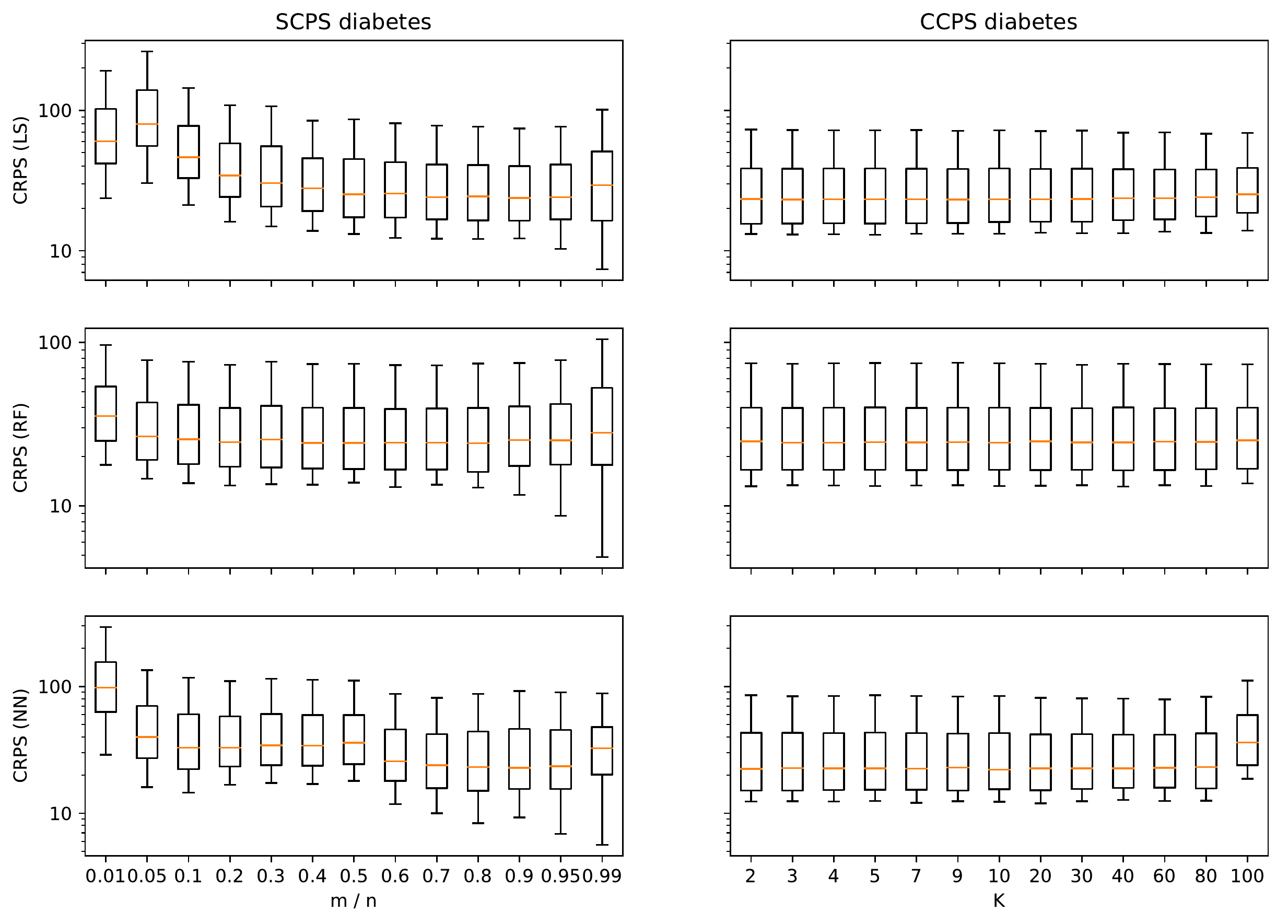}
  \end{center}
  \caption{The analogue of Figure~\ref{fig:boston} for the \texttt{Diabetes} dataset.}
  \label{fig:diabetes}
\end{figure}

The \texttt{Diabetes} dataset consists of 10 physiological measures on 442 patients,
and the label indicates disease progression after one year.
Figure~\ref{fig:diabetes} is the analogue of Figure~\ref{fig:boston}
for this dataset.
We can see the same tendencies,
with $K\in\{5,10\}$ still being reasonable numbers of folds for CCPS.

\begin{figure}[bt]
  \begin{center}
    \includegraphics[width=0.99\textwidth]{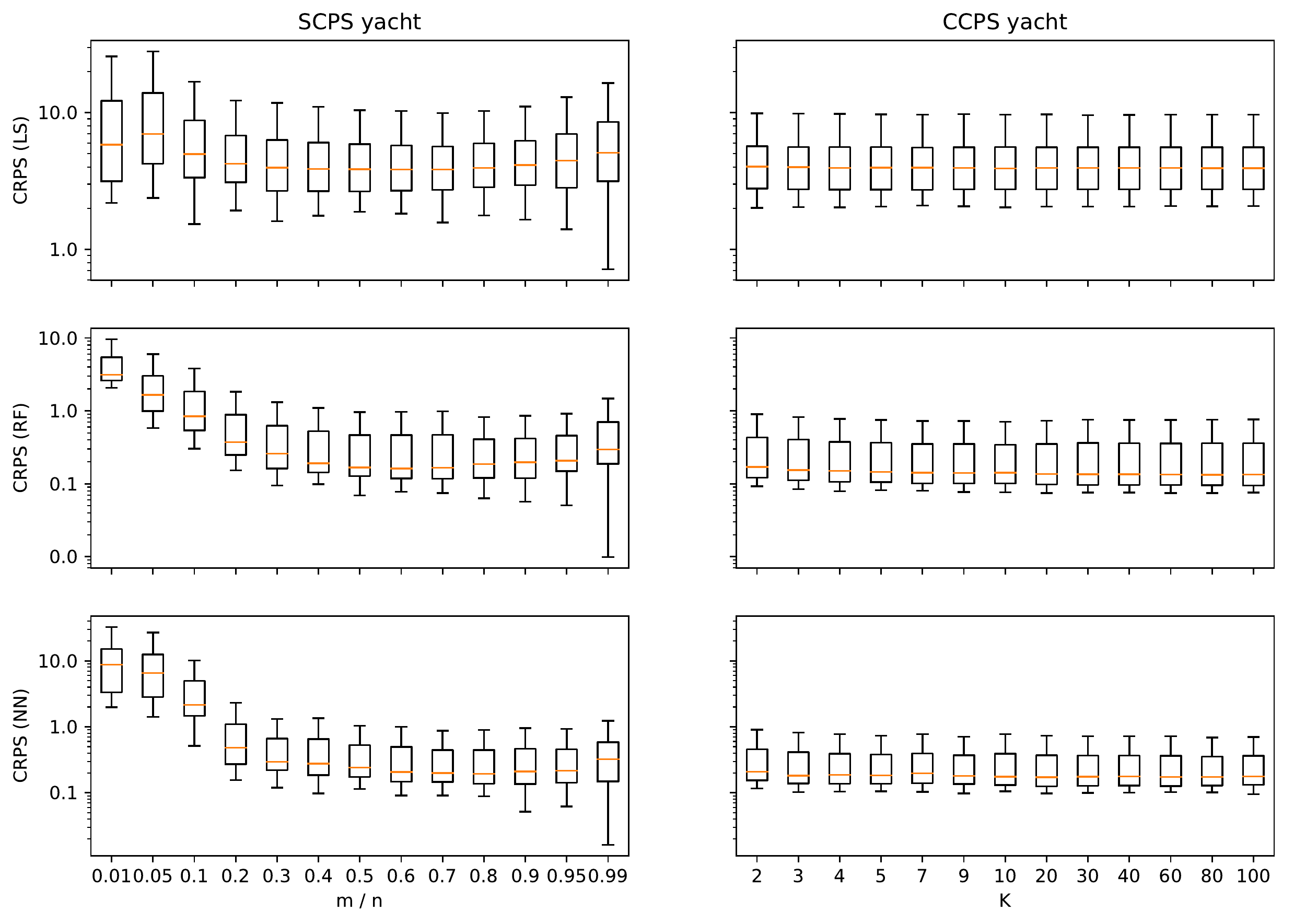}
  \end{center}
  \caption{The analogue of Figure~\ref{fig:boston} for the \texttt{Yacht Hydrodynamics} dataset.}
  \label{fig:yacht}
\end{figure}

The \texttt{Yacht Hydrodynamics} is the smallest of our datasets.
It consists of 7 attributes including the basic hull dimensions and the boat velocity for 308 experiments,
and the task is to predict the residuary resistance of sailing yachts.
Figure~\ref{fig:yacht} suggests that the behavior shown in Figures~\ref{fig:boston} and~\ref{fig:diabetes}
is in fact typical of small datasets.

\begin{figure}[bt]
  \begin{center}
    \includegraphics[width=0.99\textwidth]{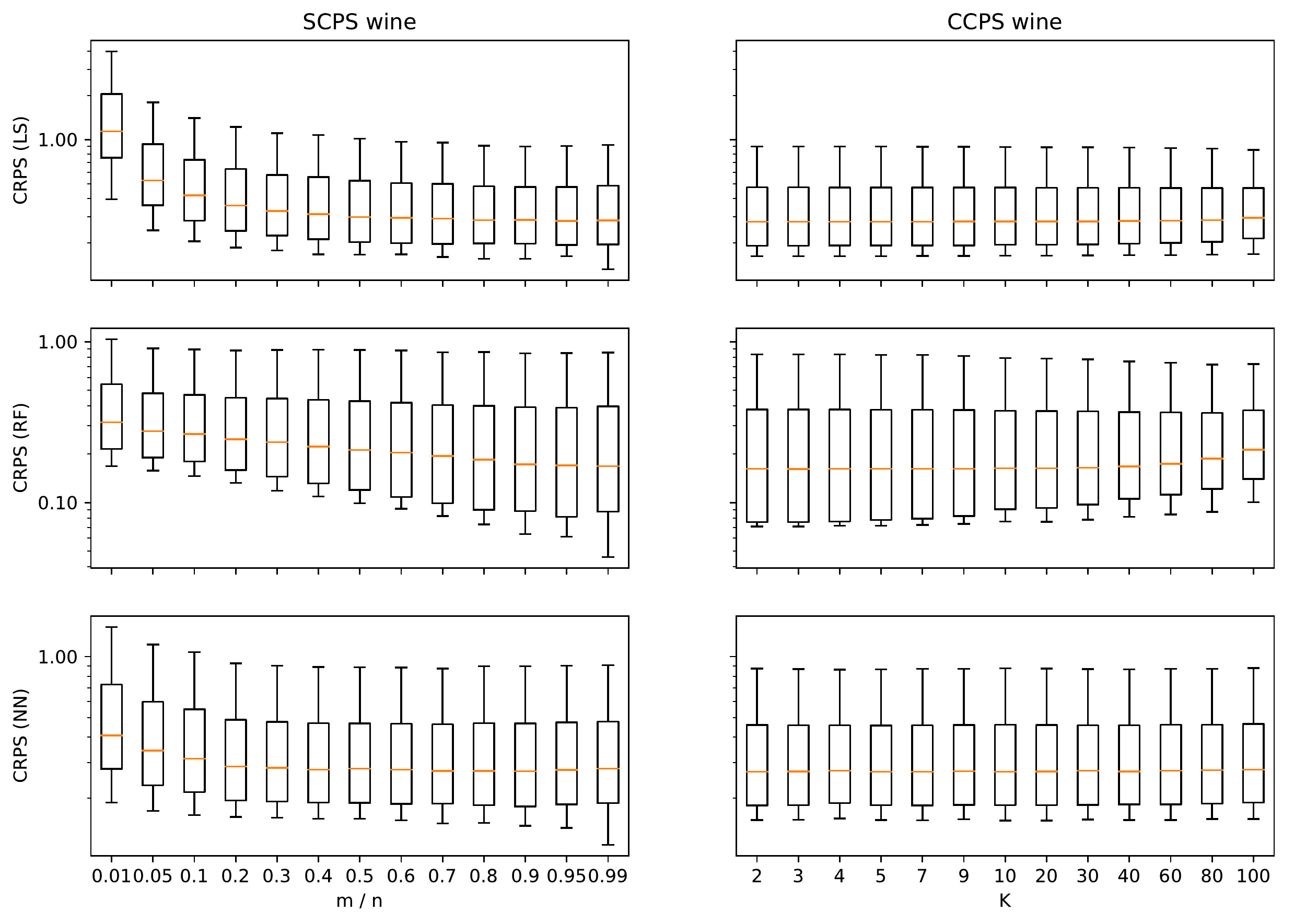}
  \end{center}
  \caption{The analogue of Figure~\ref{fig:boston} for the \texttt{Wine Quality} dataset.}
  \label{fig:wine}
\end{figure}

The \texttt{Wine Quality} dataset has information about 1599 red wines and 4898 white wines.
We merge these two groups creating another attribute taking two values,
0 for white and 1 for red.
The label is the quality of wine expressed as a score between 0 and 10.
(The most common labels are 5 and 6, labels 3 and 9 are very uncommon,
and labels 0 and 1 are absent.)

Figure~\ref{fig:wine} is qualitatively similar to Figures~\ref{fig:boston} and~\ref{fig:diabetes}.
The shape of the plots for SCPS suggests that we need a reasonable length $n-m$ of the calibration sequence,
such as 100 or 200,
since it determines the granularity of the predictive distributions:
as we have already mentioned in connection with \eqref{eq:Q-1},
the thickness of the predictive distribution is $\frac{1}{n-m+1}$.
Increasing the length of the calibration sequence further
does not improve the predictive performance significantly,
and starts hurting it when the training sequence proper becomes too short.

\begin{figure}[bt]
  \begin{center}
    \includegraphics[width=0.99\textwidth]{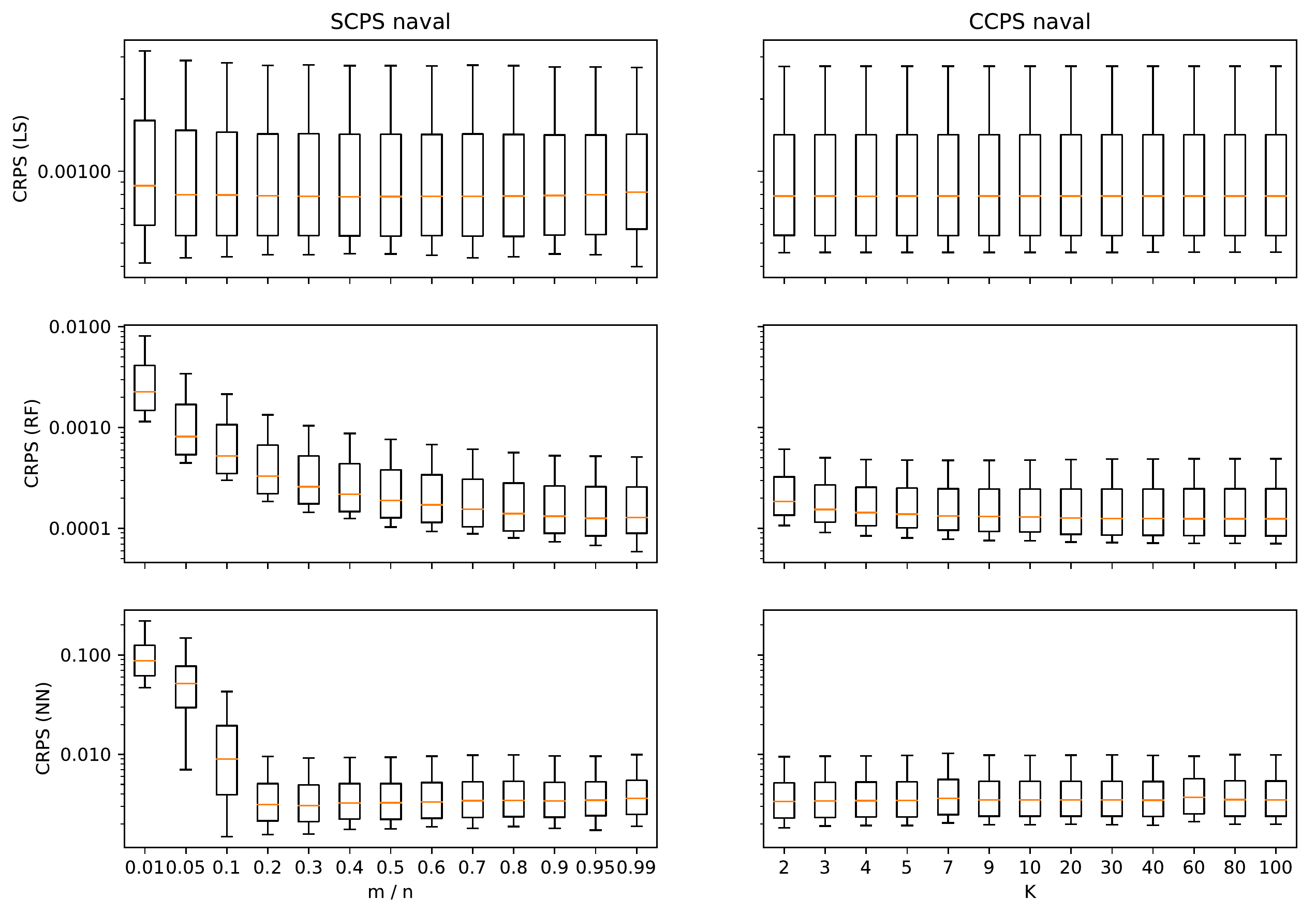}
  \end{center}
  \caption{The analogue of Figure~\ref{fig:boston} for the \texttt{Naval Propulsion} dataset.}
  \label{fig:naval}
\end{figure}

Figure~\ref{fig:naval} reports the results for the largest dataset that we use,
\texttt{Naval Propulsion}.
It contains information about 11,934 simulated experiments,
each described by 16 attributes,
and the task is to predict the Gas Turbine Compressor decay state coefficient for a propulsion plant.
Here we observe the same general behavior.

\begin{table}[bt]
  \caption{Best results for the median CRPS loss for SCPS and CCPS for the five datasets and three underlying algorithms.}

  \medskip

  \begin{center}
    \begin{tabular}{llrr}
      Dataset & underlying algorithm & SCPS    & CCPS \\
      \hline
      Boston Housing & Least Squares  & 1.726  & 1.533 \\
      Boston Housing & Random Forest  & 0.972  & 0.906 \\
      Boston Housing & Neural Network & 1.240  & 1.211 \\
      Diabetes & Least Squares        & 23.74 & 23.18 \\
      Diabetes & Random Forest        & 24.23 & 24.33 \\
      Diabetes & Neural Network       & 22.76 & 22.10 \\
      Yacht Hydrodynamics & Least Squares   & 3.840 & 3.910 \\
      Yacht Hydrodynamics & Random Forest        & 0.1615 & 0.1322 \\
      Yacht Hydrodynamics & Neural Network       & 0.1944 & 0.1725 \\
      Wine Quality & Least Squares    & 0.2810  & 0.2771 \\
      Wine Quality & Random Forest    & 0.1681  & 0.1618 \\
      Wine Quality & Neural Network   & 0.2711  & 0.2693 \\
      Naval Propulsion & Least Squares    & 0.0007812  & 0.0007866 \\
      Naval Propulsion & Random Forest    & 0.0001259  & 0.0001242 \\
      Naval Propulsion & Neural Network   & 0.003051 & 0.003360
    \end{tabular}
  \end{center}
  \label{tab:results}
\end{table}

The best results presented in Figures~\ref{fig:boston}--\ref{fig:naval}
are summarized in Table~\ref{tab:results}.
Namely, the table reports the median CRPS losses shown
in Figures~\ref{fig:boston}--\ref{fig:wine}
obtained by optimizing the parameters $m/n$ in the case of SCPS and $K$ in the case of CCPS.
In the majority of cases CCPS perform better than SCPS.
But what is even more important,
CCPS are much less sensitive to choosing their parameter $K$,
and so the best results given in Table~\ref{tab:results} are in fact typical for them.
In all our experiments, it is safe to choose any of the standard values for the number $K$ of folds
in the range from 5 to 10.

\begin{figure}[bt]
  \begin{center}
    \includegraphics[width=0.99\textwidth]{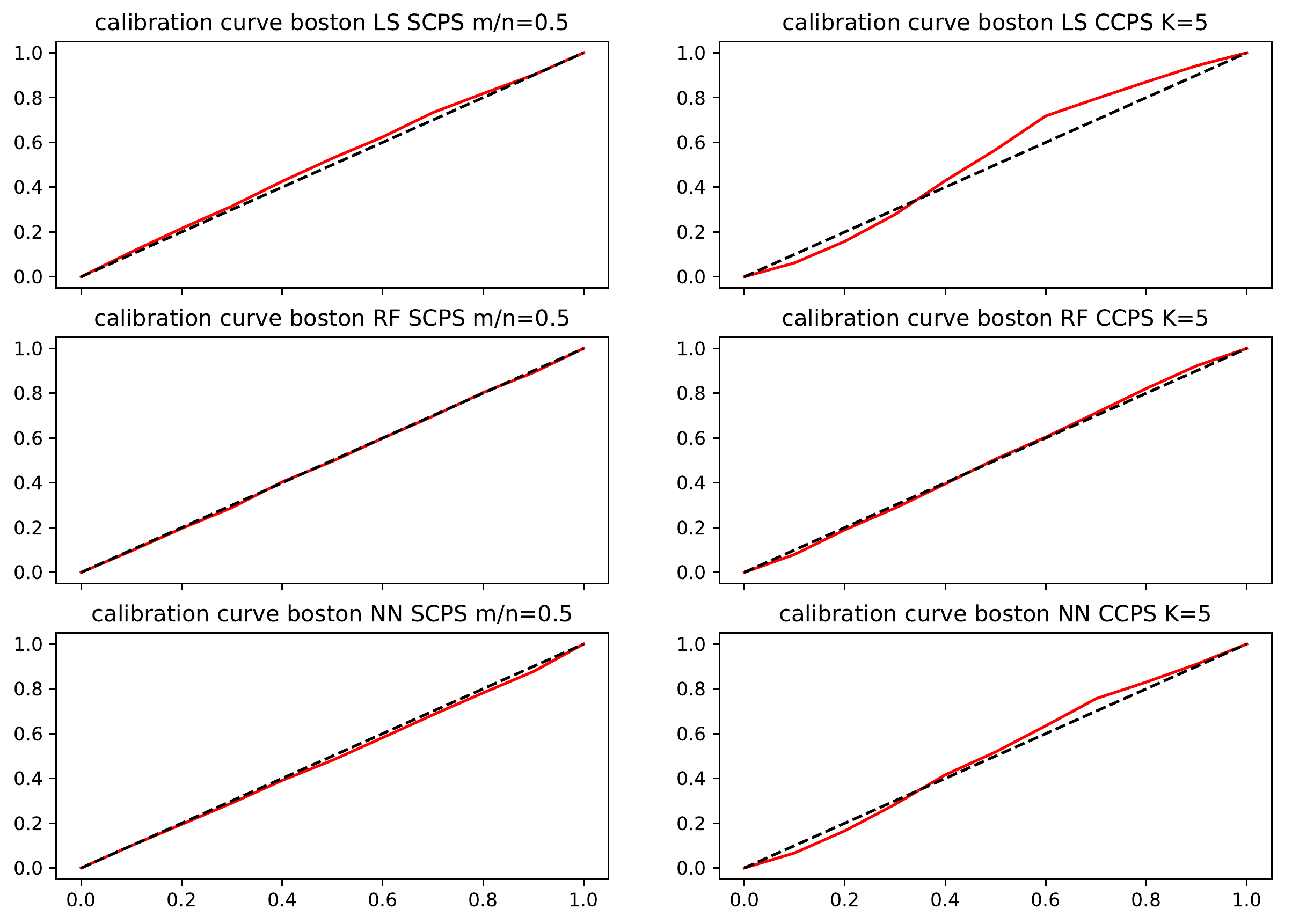}
  \end{center}
  \caption{The calibration curves
    (i.e., the distributions of $Q(z_1,\ldots,z_n,(x,y))$ over the test sequence)
    for the SCPS and CCPS on the \texttt{Boston Housing} dataset.}
  \label{fig:calibration_boston}
\end{figure}

\begin{figure}[bt]
  \begin{center}
    \includegraphics[width=0.99\textwidth]{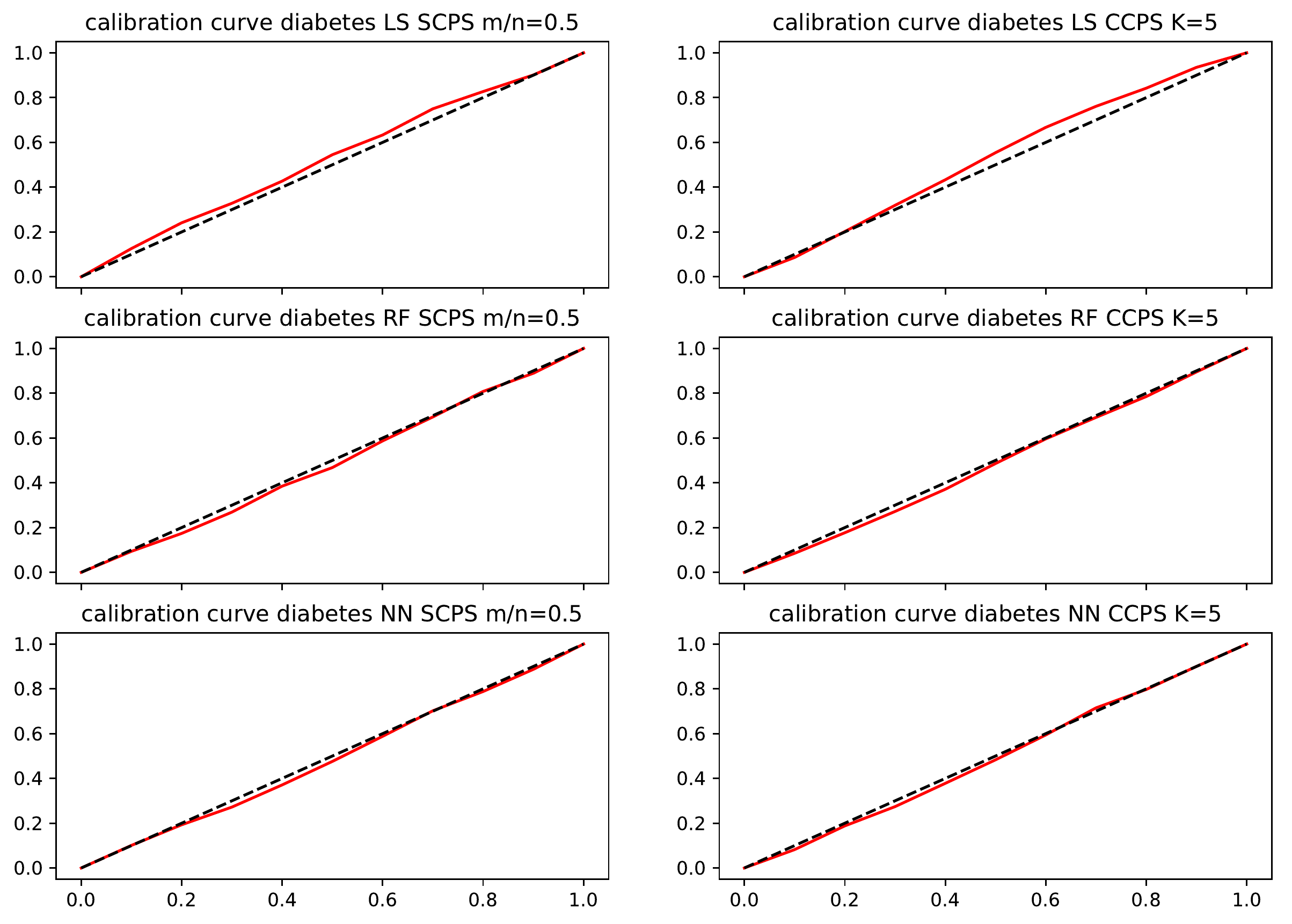}
  \end{center}
  \caption{The analogue of Figure~\ref{fig:calibration_boston} for the \texttt{Diabetes} dataset.}
  \label{fig:calibration_diabetes}
\end{figure}

\begin{figure}[bt]
  \begin{center}
    \includegraphics[width=0.99\textwidth]{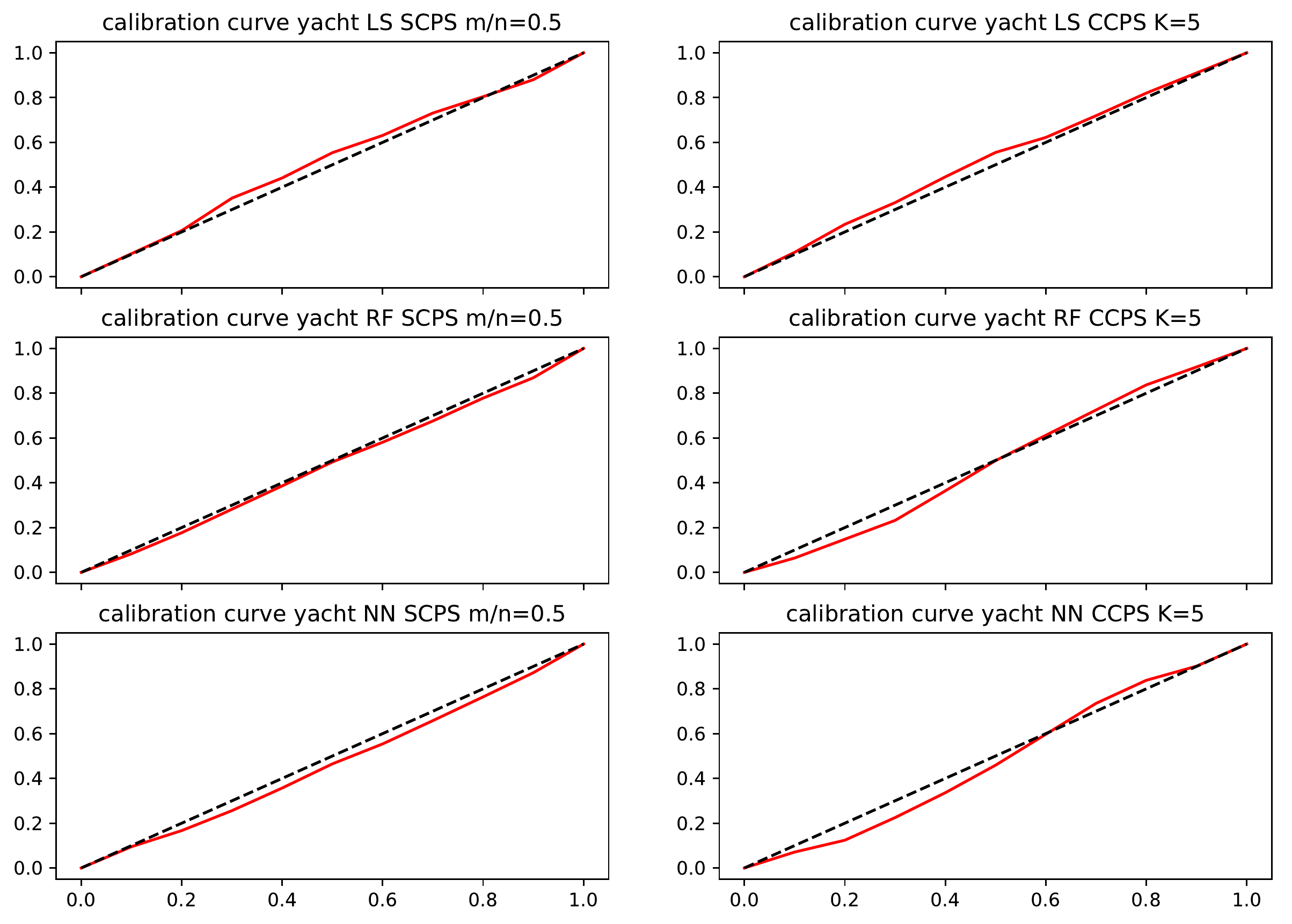}
  \end{center}
  \caption{The analogue of Figure~\ref{fig:calibration_boston} for the \texttt{Yacht Hydrodynamics} dataset.}
  \label{fig:calibration_yacht}
\end{figure}

\begin{figure}[bt]
  \begin{center}
    \includegraphics[width=0.99\textwidth]{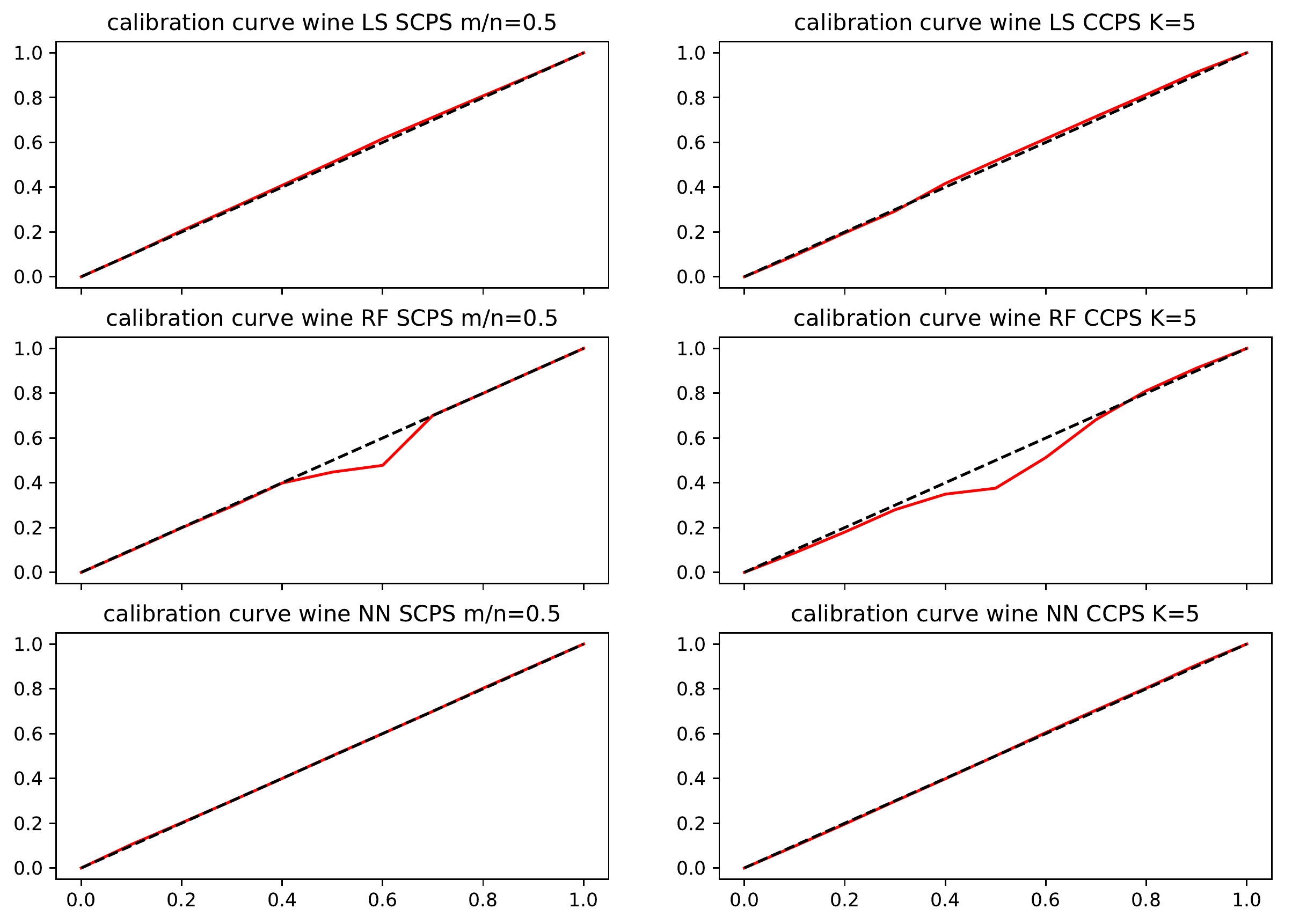}
  \end{center}
  \caption{The analogue of Figure~\ref{fig:calibration_boston} for the \texttt{Wine Quality} dataset.}
  \label{fig:calibration_wine}
\end{figure}

\begin{figure}[bt]
  \begin{center}
    \includegraphics[width=0.99\textwidth]{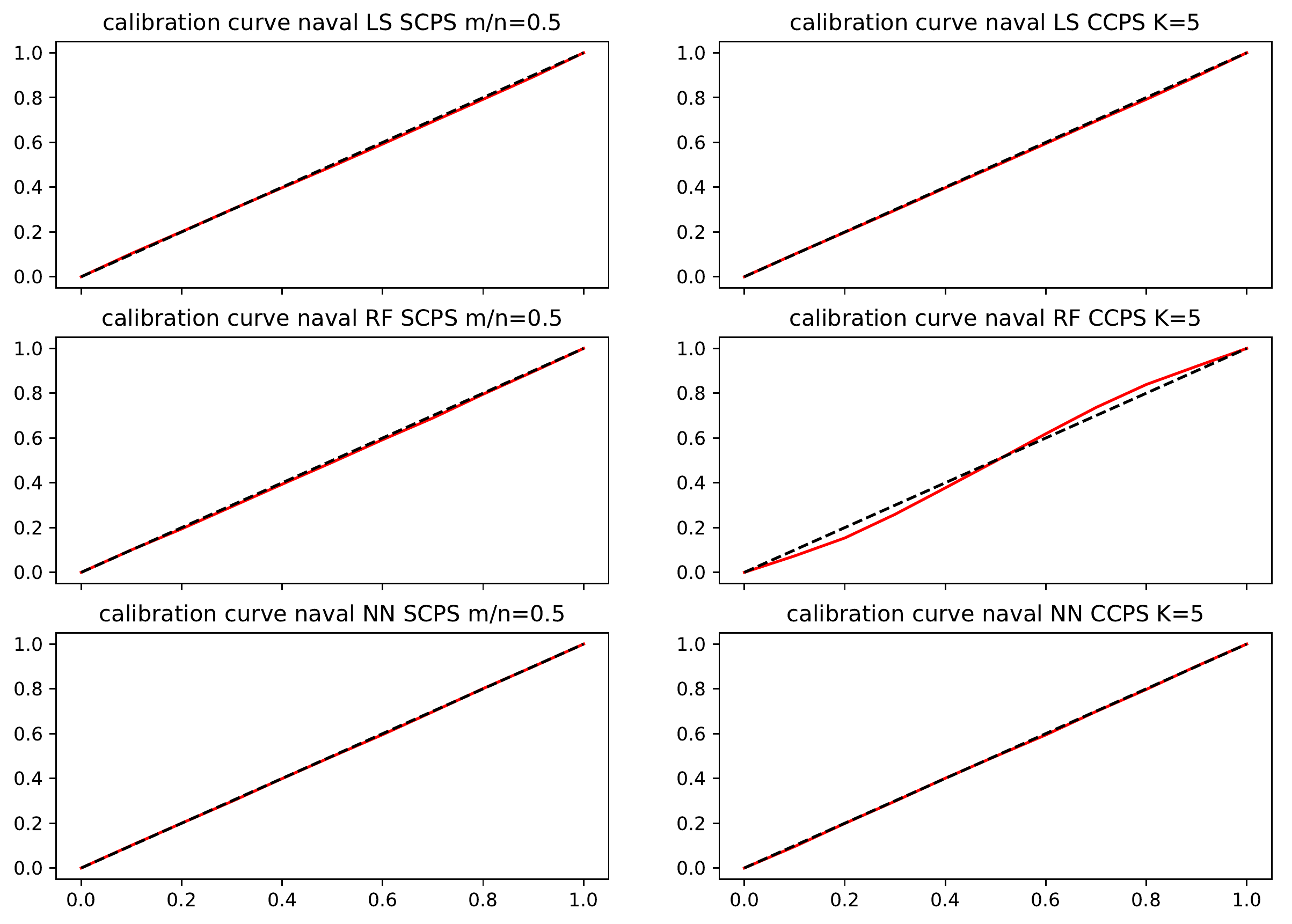}
  \end{center}
  \caption{The analogue of Figure~\ref{fig:calibration_boston} for the \texttt{Naval Propulsion} dataset.}
  \label{fig:calibration_naval}
\end{figure}

A natural question is whether the CCPS satisfy the property of validity R2 at least approximately;
remember that there are no theoretical validity results for cross-conformal predictors,
and it has been demonstrated theoretically \cite[Appendix A]{\OCMVI}
and experimentally \cite{Linusson/etal:2017} that a loss of validity is possible.
Figure~\ref{fig:calibration_boston} (right panel)
shows the distribution of the values \eqref{eq:Q-2-crisp} for \texttt{Boston Housing} and $K=5$,
where $z_1,\ldots,z_n$ is the training sequence, and $(x,y)$ range over the elements of the test sequence.
Namely, it gives the \emph{calibration curves}, which are the sets of points $(\alpha,F(\alpha))$,
$\alpha\in(0,1)$ ranging over the possible significance levels
and $F(\alpha)$ being the percentage of the values $Q(z_1,\ldots,z_n,(x,y))$
for $(x,y)$ in the test sequence that do not exceed $\alpha$.
The right panels of Figures~\ref{fig:calibration_diabetes}, \ref{fig:calibration_yacht},
\ref{fig:calibration_wine}, and~\ref{fig:calibration_naval}
are the analogues for the \texttt{Diabetes}, \texttt{Yacht Hydrodynamics},
\texttt{Wine Quality}, and \texttt{Naval Propulsion} datasets, respectively.
Under perfect validity \eqref{eq:R2} and an infinitely long test sequence,
the calibration curves should be the diagonals shown as dashed lines
on both panels of Figures~\ref{fig:calibration_boston}--\ref{fig:calibration_naval};
the actual calibration curves are fairly close.
The calibration curves for other $K$ are roughly similar.
As mentioned earlier, we also give calibration results for SCPS
(in the left panels and with $m/n\approx0.5$).

Not only is the efficiency of the CCPS with respect to the CRPS loss
better than that of the SCPS, it can also be argued that the CCPS may be safer
from the point of view of validity.
Suppose that, for some reason, we would like to avoid randomization
and use \eqref{eq:Q-1-crisp} (in the case of SCPS) or \eqref{eq:Q-2-crisp} (in the case of CCPS)
instead of \eqref{eq:Q-1} or \eqref{eq:Q-2}, respectively.
The CCPS is still empirically valid in our experiments,
even in the extreme case of $K=100$.
On the other hand, when using \eqref{eq:Q-1-crisp} in place of \eqref{eq:Q-1},
the SCPS lose not only theoretical but also empirical validity.
For example, for \texttt{Boston Housing} and $m/n=0.99$
(the right end of the horizontal axis in the left panel of Figure~\ref{fig:boston}),
the length of the calibration sequence is $4$,
and so the empirical predictive distribution \eqref{eq:Q-1-crisp}
only takes values in $\{0,0.25,0.5,0.75,1\}$;
the distribution of its values at the true labels is clearly very different from being uniform.

\section{Universal consistency of predictive systems}
\label{sec:univ}

The conference version \cite{Vovk/etal:2018COPA} of this paper
was published in the proceedings of COPA 2018,
which also contained a paper \cite{Nouretdinov/etal:2018COPA}
that adapted Venn prediction for producing predictive distributions.
In this and next sections we will analyze
the asymptotic performance of the two approaches to predictive distributions,
using conformal prediction \cite{Vovk/etal:2018COPA}
and using Venn prediction \cite{Nouretdinov/etal:2018COPA}.
Our conclusion is that, as implemented in those papers,
both approaches are very restrictive.
But whereas the approach based on conformal prediction can be easily rescued,
fixing the approach based on Venn prediction might require
sacrificing computational efficiency.
In this section we discuss the former approach.

Informally, an RPS is universally consistent if it gives the true predictive distribution
in the limit,
and a class of RPS is universal if it contains such an RPS.
The following formalization is given in \cite{\OCMXVIII}
and its idea goes back to Belyaev's work
(see, e.g., \cite{Belyaev/Sjostedt:2000}).
\begin{definition}\label{def:Belyaev}
  An RPS $Q$ is \emph{consistent} for a probability measure $P$ on $\mathbf{Z}$
  if, for any bounded continuous function $f:\R\to\R$,
  \begin{equation}\label{eq:Belyaev}
    \int f \dd Q_n
    -
    \Expect(f\mid x_{n+1})
    \to
    0
    \qquad
    (n\to\infty)
  \end{equation}
  in probability, where:
  \begin{itemize}
  \item
    $Q_n$ is the predictive distribution function \eqref{eq:Q} for the label of $x_{n+1}$
    based on the training sequence $z_1,\dots,z_n$;
    the integral $\int f \dd Q_n$ is not quite standard since we did not require $Q_n$ to be exactly a distribution function,
    and we understand it as $\int f \dd \bar Q_n$
    with the measure $\bar Q_n$ on $\R$ defined by $\bar Q_n((u,v]):=Q_n(v+)-Q_n(u+)$
    for any interval $(u,v]$ of this form in $\R$;
  \item
    $\Expect(f\mid x_{n+1})$ is the conditional expectation for $f(y)$ given $x=x_{n+1}$
    assuming $(x,y)\sim P$
    (we fix a version of the conditional expectation);
  \item
    the data-generating and coin-tossing mechanisms are $z_i=(x_i,y_i)\sim P$, $i=1,\dots,n+1$,
    and $\tau\sim U$, assumed all independent.
  \end{itemize}
  We say that $Q$ is \emph{universally consistent}
  if it is consistent for any probability measure $P$ on $\mathbf{Z}$.
  A class of RPS is \emph{universal} if it contains a universally consistent RPS.
\end{definition}

The requirement of a class of RPS being universal
means that it does not impose insurmountable limits to getting the data-generating distribution right.

We will also apply Definition~\ref{def:Belyaev} to predictive systems
that are not required to satisfy the validity condition R2 in Definition~\ref{def:RPS}
(such as CCPS and predictive systems based on Venn prediction).
The following theorem assumes that the notions of SCPS and CCPS
have been slightly extended by allowing randomized conformity measures
(see, e.g., \cite[Section 9]{\OCMXVIII} for details).

\begin{theorem}\label{thm:universal}
  The class of SCPS is universal.
  The class of CCPS is also universal.
\end{theorem}

\begin{proof}
  The class of SCPS being universal
  is a simplified version of \cite[Theorem~31]{\OCMXVIII}.
  For a fixed $K$,
  a $K$-fold CCPS outputs predictive distribution functions
  within $O(1/n)$ of the average of the predictive distribution functions
  output by the component SCPS
  (see \eqref{eq:modified-mean}),
  which immediately implies that the class of CCPS
  is also universal.
\end{proof}

Theorem~\ref{thm:universal} says that, in principle,
conformal predictive systems can adapt to any data-generating distribution.
However, specific conformal predictive systems considered in literature
are often not universally consistent.
This is particularly true for predictive systems based on full
(rather than split or cross-) conformal prediction,
where computational efficiency imposes severe restrictions
on the underlying algorithm.

As discussed earlier,
an example of a non-universal class of RPS is provided by \cite{\OCMXVII}:
they are based on the method of Least Squares and therefore far from being universal.
The extension of Least Squares to Kernel Ridge Regression given in \cite{\OCMXX}
still does not produce universality, even for universal kernels:
the Kernel Ridge Regression Prediction Machine introduced in \cite{\OCMXX} is not universal
since the shape of its predictive distributions is not tailored
to a specific test observation \cite[Section~7]{\OCMXX}.

Conformal predictive systems based on \eqref{eq:example}
are also not universal:
they allow any shape of the asymptotic predictive distribution function,
but this shape is adapted to the test object at hand
only by shifting and scaling it
(by \emph{shifting} a distribution function $F$
we mean replacing it by the distribution function $y\mapsto F(y-c)$ for some $c\in\R$,
and by \emph{scaling} we mean replacing it by $y\mapsto F(y/\sigma)$ for some $\sigma>0$).
The class \eqref{eq:example-special} of split and cross-conformity measures
considered in the experimental section is even more restrictive:
the asymptotic shape of the predictive distribution function
is adapted to the test object at hand only by a shift.
We will state in detail only the claim about the class \eqref{eq:example-special},
since it was the main class used in this paper,
and it is also the class used in \cite{\OCMXVII} and \cite{\OCMXX}.

Let us call, for want of a better name,
split conformity measures of the form \eqref{eq:example-special} \emph{simple}.
Suppose the probability measure $P$ generating the observations $(x,y)$
satisfies $\Expect\left|y\right|<\infty$ for $(x,y)\sim P$.
Set $A(x,y):=y-\hat y$,
where $\hat y:=\Expect(y\mid x)$ is a fixed version of the conditional expectation of $y$ given $x$
for $(x,y)\sim P$.
The \emph{ideal simple conformal predictive system} (ISCPS) for $P$
is defined as
\begin{multline}\label{eq:ISCPS}
  Q(z_1,\ldots,z_n,(x,y),\tau)
  :=
  \frac{1}{n+1}
  \left|\left\{i=1,\ldots,n\mid A(x_i,y_i)<A(x,y)\right\}\right|\\
  +
  \frac{\tau}{n+1}
  \left|\left\{i=1,\ldots,n\mid A(x_i,y_i)=A(x,y)\right\}\right|
  +
  \frac{\tau}{n+1},
\end{multline}
where $x$ is the test object.
The intuition behind this definition
is that we are given $P$ in advance and, therefore,
can use the whole training sequence as the calibration sequence;
a training sequence proper is not needed as $A$ is already the ideal simple conformity measure.
An ISCPS is an idealization of SCPS corresponding to an infinitely long training sequence proper
(allowing a perfect estimate of $\Expect(y\mid x)$).
Since CCPS are essentially combinations of SCPS,
our conclusions will also be applicable to CCPS.

Remember that the Kolmogorov distance between distribution functions $F$ and $G$ is
\[
  K(F,G)
  :=
  \sup_{u\in\R}
  \left|
    F(u) - G(u)
  \right|.
\]
Modify it by setting
\[
  K'(F,G)
  :=
  \adjustlimits
  \inf_{c\in\R}
  \sup_{u\in\R}
  \left|
    F(u-c) - G(u)
  \right|.
\]
This is not a metric anymore: $K'(F,G)=0$ only means that $F$ and $G$ coincide
to within a shift left or right.
The following proposition spells out the observation above that,
in the case of a simple conformity measure,
the asymptotic shape of the predictive distribution function
is adapted to the test object at hand only by a shift.

\begin{proposition}\label{prop:simple}
  Let $Q$ be an ISCPS.
  For all $n$, all $z_1,\dots,z_n\in\mathbf{Z}$,
  all $x,x'\in\mathbf{X}$, and all $\tau\in[0,1]$,
  \[
    K'
    \left(
      Q(z_1,\dots,z_n,(x,\cdot),\tau),
      Q(z_1,\dots,z_n,(x',\cdot),\tau)
    \right)
    =
    0.
  \]
\end{proposition}

\begin{proof}
  Since $A(x,y)=y-\hat y$,
  $Q(z_1,\dots,z_n,(x,\cdot),\tau)$ is of the form $F(\cdot-\hat y)$
  and $Q(z_1,\dots,z_n,(x',\cdot),\tau)$ is of the form $F(\cdot-\hat y')$
  for some numbers $\hat y$ and $\hat y'$ and some function $F$
  (see \eqref{eq:ISCPS}).
  Therefore, they are shifts of each other.
\end{proof}

Proposition~\ref{prop:simple} will remain true if $\hat y:=\Expect(y\mid x)$
in the definition of a simple conformity measure
is replaced by $\hat y:=f(x)$ for any function $f:\mathbf{X}\to\R$.

The idealized version of the split conformity measure \eqref{eq:example} is
\begin{equation}\label{eq:A}
  A(x,y)
  :=
  \frac{y-f(x)}{\sigma(x)}
\end{equation}
for some positive function $\sigma:\mathbf{X}\to(0,\infty)$.
Let us modify further the Kolmogorov distance by setting
\[
  K''(F,G)
  :=
  \adjustlimits
  \inf_{c\in\R,\sigma>0}
  \sup_{u\in\R}
  \left|
    F\left(\frac{u-c}{\sigma}\right) - G(u)
  \right|.
\]
Proposition~\ref{prop:simple} will continue to hold
if we allow ISCPS to use idealized split conformity measures \eqref{eq:A}
and replace $K'$ by $K''$.

Using split and cross-conformal predictive systems
rather than full conformal predictive systems
makes it much easier to design adaptive conformity measures.
One possibility is to use the Nadaraya--Watson estimate
(introduced by \cite{Nadaraya:1964} and \cite{Watson:1964} in the case of regression
and \cite{Rosenblatt:1969} in the case of density estimation)
\begin{equation}\label{eq:NW}
  F(y\mid x)
  =
  \frac
  {
    \sum_{i=1}^m
    \Sigma
    \left(
      \frac{y-y_i}{h_y}
    \right)
    K
    \left(
      \frac{x-x_i}{h_x}
    \right)
  }
  {
    \sum_{i=1}^m
    K
    \left(
      \frac{x-x_i}{h_x}
    \right)
  }
\end{equation}
of the conditional distribution function
for computing the conformity score of $(x,y)$
given $(x_1,y_1),\ldots,(x_m,y_m))$.
The parameters of the estimator~\eqref{eq:NW}
are a distribution function $\Sigma$
(e.g., the Heaviside step function or a smooth one,
such as the sigmoid $\Sigma(u):=1/(1+e^{-u})$,
in which case there is a unique solution to the equations
in Algorithms~\ref{alg:SCPS} and~\ref{alg:CCPS}),
a kernel $K$ (such as the Gaussian $K(u):=\exp(-u^2/2)$),
and bandwidths $h_x>0$ and $h_y>0$.
This is the topic of \cite{\OCMXXIII}.

\section{Split Venn--Abers predictive systems}
\label{sec:VA}

In this section we discuss an alternative to RPS
introduced in \cite{Nouretdinov/etal:2018COPA}
and based on Venn prediction.
We will obtain a modification of RPS defined as follows
(cf.\ Definition~\ref{def:RPS}).

\begin{definition}
  A function $Q:\mathbf{Z}^{n+1}\times\{0,1\}\to[0,1]$ is called an \emph{imprecise predictive system} (IPS)
  if it satisfies the following two requirements:
  \begin{itemize}
  \item[i]
    For each training sequence $(z_1,\ldots,z_n)\in\mathbf{Z}^n$ and each test object $x\in\mathbf{X}$,
    the function $Q(z_1,\ldots,z_n,(x,y),\tau)$ is monotonically increasing both in $y$ and in $\tau$.
    In other words, for either $\tau\in\{0,1\}$,
    the function~\eqref{eq:function} is monotonically increasing,
    and for each $y\in\R$,
    \[
      Q(z_1,\ldots,z_n,(x,y),0)
      \le
      Q(z_1,\ldots,z_n,(x,y),1).
    \]
  \item[ii]
    For each training sequence $(z_1,\ldots,z_n)\in\mathbf{Z}^n$ and each test object $x\in\mathbf{X}$,
    we have~\eqref{eq:y-minus-infty} and~\eqref{eq:y-infty}.
  \end{itemize}
\end{definition}
As compared with Definition~\ref{def:RPS}, we drop the validity condition~R2
(which may be stated separately in some form when needed).

We start the definition of a Venn-type IPS from an analogue of a conformity measure.
\begin{definition}
  A \emph{regressor} is a family of measurable functions $A_m:\mathbf{Z}^{m}\times\mathbf{X}\to\R$, $m=1,2,\dots$.
\end{definition}
\noindent
The intention is that $A_m(z_1,\ldots,z_{m},x)$ is the prediction for the label of $x$
computed from $z_1,\ldots,z_m$ as training sequence.
As before, we drop the lower index $m$ in $A_m$.
Now we can define a new kind of predictive systems.
\begin{definition}\label{def:SVAPS}
  Split the training sequence $z_1,\ldots,z_n$ into two parts:
  the \emph{training sequence proper} $z_1,\ldots,z_m$ and the \emph{calibration sequence} $z_{m+1},\ldots,z_n$.
  Suppose we are given a test object $x$ and a possible label $y\in\R$ for it.
  The output $Q(z_1,\ldots,z_n,(x,y),\tau)$, $\tau\in\{0,1\}$,
  of the \emph{split Venn--Abers predictive system of type $T$}
  determined by the regressor $A$,
  where $T\in\{1,2,3\}$,
  is defined as follows:
  \begin{itemize}
  \item
    set $s_i:=A(z_1,\ldots,z_m,x_i)$ for $i=1,\ldots,n$
    and set $s:=A(z_1,\ldots,z_m,x)$;
  \item
    fit an isotonic regression $g:\R\to\R$ to the training sequence $(s_i,y^*_i)$
    extended by adding $(s,\tau)$,
    where
    \begin{equation}\label{eq:y-star}
      y^*_i
      :=
      \begin{cases}
        0 & \text{ if $y_i\le y$}\\
        1 & \text{otherwise}
      \end{cases}
    \end{equation}
    and the range of $i$ is
    \begin{equation*}
      i
      =
      \begin{cases}
        m+1,\ldots,n & \text{if $T=1$}\\
        1,\ldots,m & \text{if $T=2$}\\
        1,\ldots,n & \text{if $T=3$};
      \end{cases}
      \end{equation*}
      the corresponding optimization problem is
      \begin{equation}\label{eq:optimization}
        (\tau - g(s))^2
        +
        \sum_i
        (y_i^* - g(s_i))^2
        \to
        \min
      \end{equation}
      under the restriction that $g$ is monotonically increasing;
  \item
    set
    \begin{multline}\label{eq:problem}
      Q(z_1,\ldots,z_n,(x,y),\tau)
      :=
      1
      -{}\\
      \frac
        {\left|\left\{i=m+1,\ldots,n\mid g(s_i)=g(s), y_i^*=1\right\}\right|+\tau}
        {\left|\left\{i=m+1,\ldots,n\mid g(s_i)=g(s)\right\}\right|+1}.
    \end{multline}
  \end{itemize}
  An IPS is a \emph{split Venn--Abers predictive system} (SVAPS) of type $T$
  if it is the split Venn--Abers predictive system determined by some regressor.
\end{definition}

Intuitively, a SVAPS tries to answer the question whether the label of the test object $x$ exceeds $y$
based on the answers \eqref{eq:y-star} for training objects;
the answer is given by the fraction in \eqref{eq:problem}.
Notice that:
\begin{itemize}
\item
  SVAPS satisfy the definition of an IPS.
  Indeed, suppose, e.g., that $y$ is sufficiently large (the case $y\to\infty$).
  Then we will have $y^*_i=0$ for all $i$.
  For $\tau=0$ we will have $g=0$ and the fraction in \eqref{eq:problem} will be 0.
  The argument for the case $y\to-\infty$ is analogous.
\item
  The integral $\int f \dd Q_n$ in \eqref{eq:Belyaev}
  may depend on $\tau$,
  and in the definitions of consistency and universality for SVAPS
  we require that \eqref{eq:Belyaev} hold for either value of $\tau$.
\end{itemize}

SVAPS of type 1 were introduced in \cite[Section~3.1.2]{Nouretdinov/etal:2018COPA}
as the most direct application of the Venn--Abers methodology \cite{\OCMVII,\OCMXIII}
to the problem of probabilistic regression.
Since the arguments $s_i$ and $s$ of the function $g$ in~\eqref{eq:problem}
are also used in the condition~\eqref{eq:optimization},
the expressions $g(s_i)$ and $g(s)$ in~\eqref{eq:problem}
are determined uniquely;
therefore, the definition~\eqref{eq:problem} of type 1 SVAPS is unambiguous.

SVAPS of type~2 were introduced in \cite[Section~3.1.3]{Nouretdinov/etal:2018COPA}
as a computational simplification of SVAPS of type~1;
additionally, \cite[Section~3.1.3]{Nouretdinov/etal:2018COPA} removes
the first addend in \eqref{eq:optimization}
(which is the key step in achieving computational efficiency).
For this version of SVAPS, the arguments $s_i$ and $s$ of the function $g$ in~\eqref{eq:problem}
are not necessarily used in the condition~\eqref{eq:optimization},
and so the expressions $g(s_i)$ and $g(s)$ in~\eqref{eq:problem}
may not be determined uniquely.
Somewhat arbitrarily,
we may define $g(t)$, for any $t\in\R$,
as $g(s_j)$ where $s_j$ is the nearest neighbour to $t$
among $s_i$, $i\in\{1,\ldots,m\}$;
in the case of ties, we choose $j$ as small as possible.
This makes the definition~\eqref{eq:problem} of type 2 SVAPS also unambiguous.

SVAPS of type~3 are a natural combination of SVAPS of types~1 and~2;
they are similar to SVAPS of type~1
in that the definition~\eqref{eq:problem} is for them unambiguous.

The validity guarantees for SVAPS are very different
from those that we have for CPS and SCPS;
see, e.g., \cite[Appendix~B]{\OCMXVIII}.

The following simple example illustrates severe restrictions of SVAPS
(of any type),
even when the training sequence is very long,
stemming from the score $A(z_1,\ldots,z_m,x)$ being just one number.

\begin{figure}[bt]
  \begin{center}
    \includegraphics[width=0.49\textwidth]{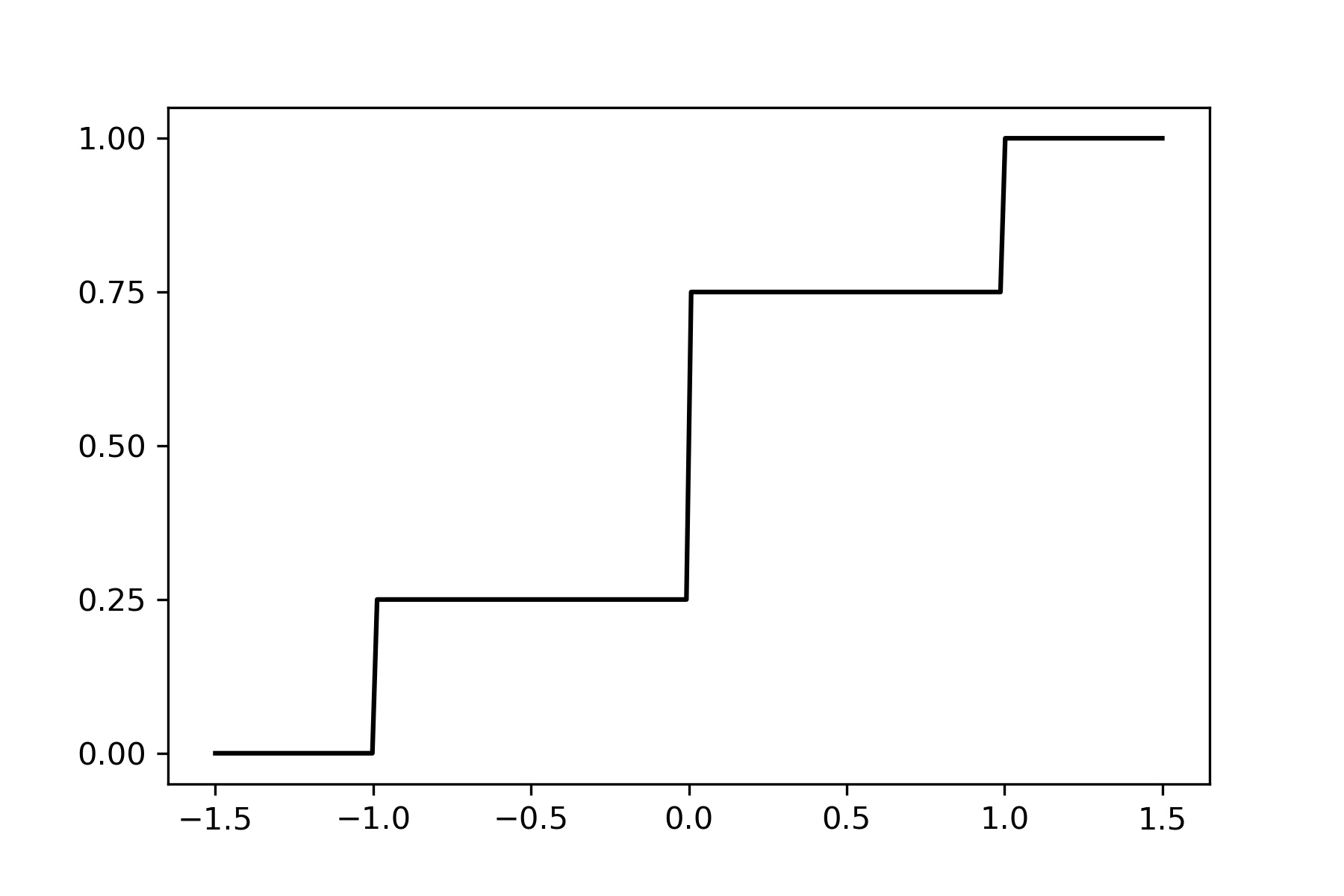}
  \end{center}
  \caption{The asymptotic predictive distribution produced by SVAPS for any $x$
    in Example~\ref{ex:problem}.}
  \label{fig:example}
\end{figure}

\begin{example}\label{ex:problem}
  The true distribution generating the observations $(x,y)$
  produces $x=0$ and $x=1$ with equal probabilities.
  Given $x=0$, we have $y=0$ with probability 1.
  Given $x=1$, we have $y=-1$ or $y=1$ with equal probabilities.
  Let us check that SVAPS are not consistent,
  even for the ideal regressor $A(z_1,\ldots,z_m,x):=0$.
  In the notation of Definition~\ref{def:SVAPS},
  we have $s=0$ and $s_i=0$ for $i=1,\ldots,n$.
  The asymptotic predictive distribution is shown in Figure~\ref{fig:example},
  concentrated on $\{-1,0,1\}$,
  and assigns probabilities $1/4$, $1/2$, and $1/4$
  to  $-1$, $0$, and $1$, respectively.
  It is very poor;
  the expected CRPS (as defined in Section~\ref{sec:CRPS}) for it is $3/8$
  instead of the ideal $1/4$.
\end{example}

Example~\ref{ex:problem} makes it plausible that the SVAPS are not universal;
now we state this formally.

\begin{proposition}\label{prop:problem}
  The SVAPS are not universal.
\end{proposition}

\begin{proof}
  Let the true probability measure be the one described in Example~\ref{ex:problem}.
  Let $n\to\infty$ and set $m:=\lfloor n/2\rfloor$.
  Set
  \begin{align*}
    A_0 &:= A(z_1,\ldots,z_m,0) \\
    A_1 &:= A(z_1,\ldots,z_m,1)
  \end{align*}
  (in Example~\ref{ex:problem} we only considered the case $A_1=A_0=0$).
  If $A_1=A_0$, we are in the situation of Example~\ref{ex:problem};
  see Figure~\ref{fig:example}.
  For a continuous $f:\R\to[0,1]$ satisfying
  \begin{equation}\label{eq:f}
    f(u)
    =
    \begin{cases}
      1 & \text{if $y\ge0.6$}\\
      0 & \text{if $y\le0.4$}
    \end{cases}
  \end{equation}
  we will have
  \begin{equation}\label{eq:f-1}
    \lim_{\substack{x_{n+1}=0\\n\to\infty}}
    \int f \dd Q_n
    =
    \frac14
    \ne
    0
    =
    \lim_{\substack{x_{n+1}=0\\n\to\infty}}
    \Expect(f\mid x_{n+1})
    \qquad
    \text{a.s.}
  \end{equation}
  and
  \begin{equation}\label{eq:f-2}
    \lim_{\substack{x_{n+1}=1\\n\to\infty}}
    \int f \dd Q_n
    =
    \frac14
    \ne
    \frac12
    =
    \lim_{\substack{x_{n+1}=1\\n\to\infty}}
    \Expect(f\mid x_{n+1})
    \qquad
    \text{a.s.}
  \end{equation}

  \begin{figure}[bt]
    \begin{center}
      \includegraphics[width=0.49\textwidth]{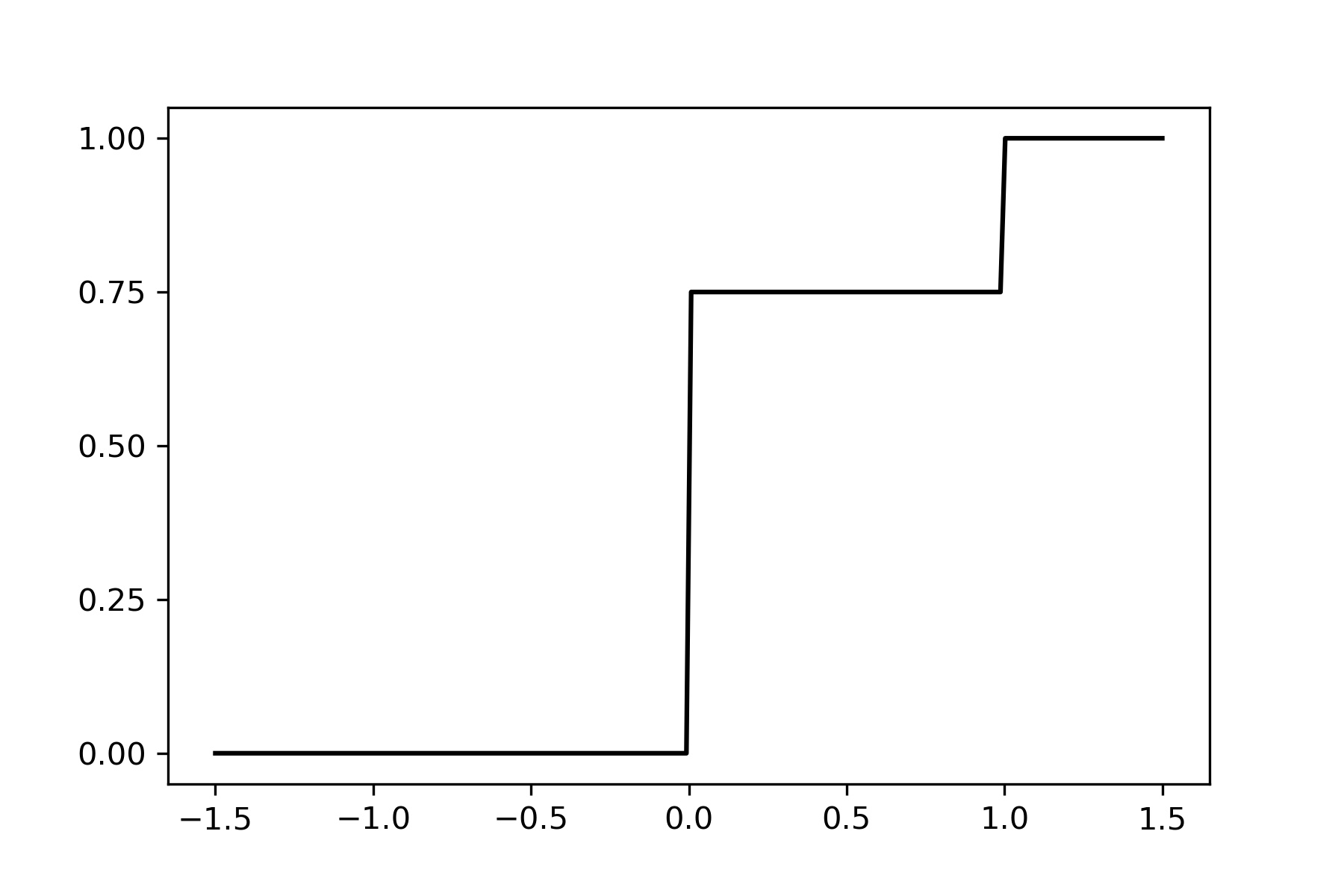}
      \hfil
      \includegraphics[width=0.49\textwidth]{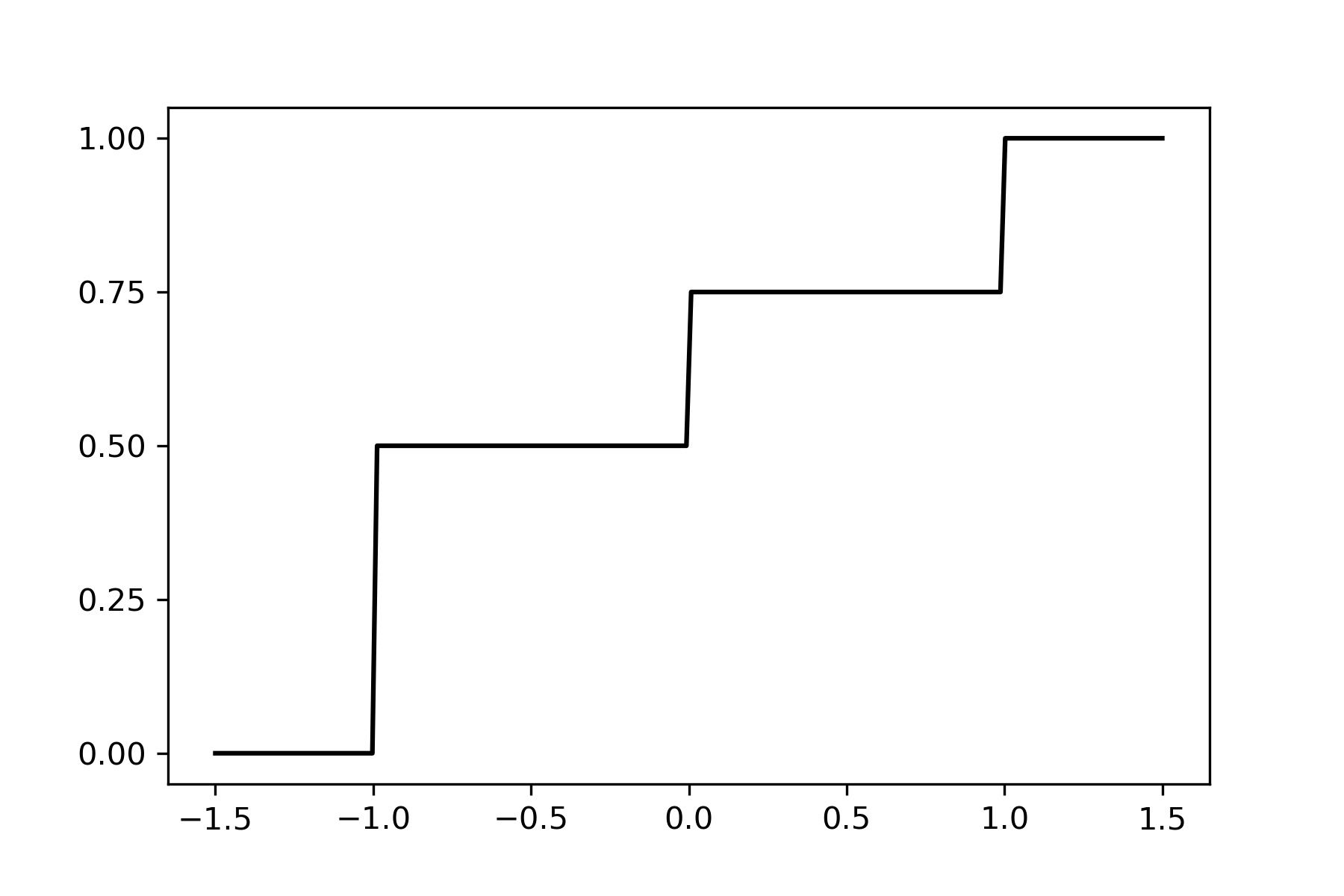}
    \end{center}
    \caption{The asymptotic predictive distributions produced by SVAPS
      when $A_1<A_0$
      for $x=0$ (left panel) and $x=1$ (right panel).}
    \label{fig:10}
  \end{figure}

  If $A_1<A_0$, the predictive distributions are as shown in Figure~\ref{fig:10}
  (the weights for $-1$, $0$, and $1$ are $0$, $3/4$, and $1/4$, respectively, when $x=0$,
  and $1/2$, $1/4$, and $1/4$, respectively, when $x=1$).
  Taking the same function $f$, satisfying \eqref{eq:f},
  we will still have \eqref{eq:f-1} and \eqref{eq:f-2}.

  \begin{figure}[bt]
    \begin{center}
      \includegraphics[width=0.49\textwidth]{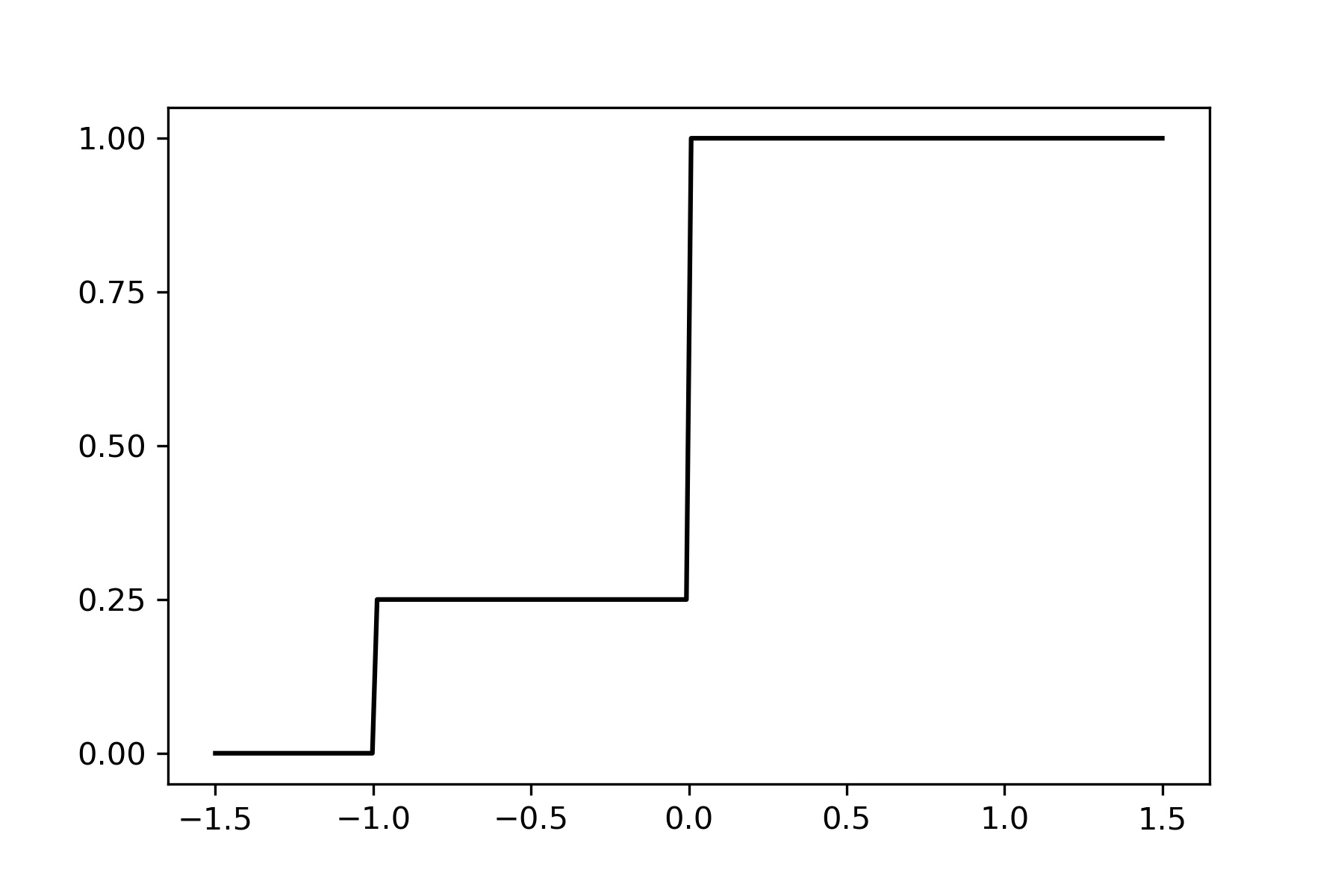}
      \hfil
      \includegraphics[width=0.49\textwidth]{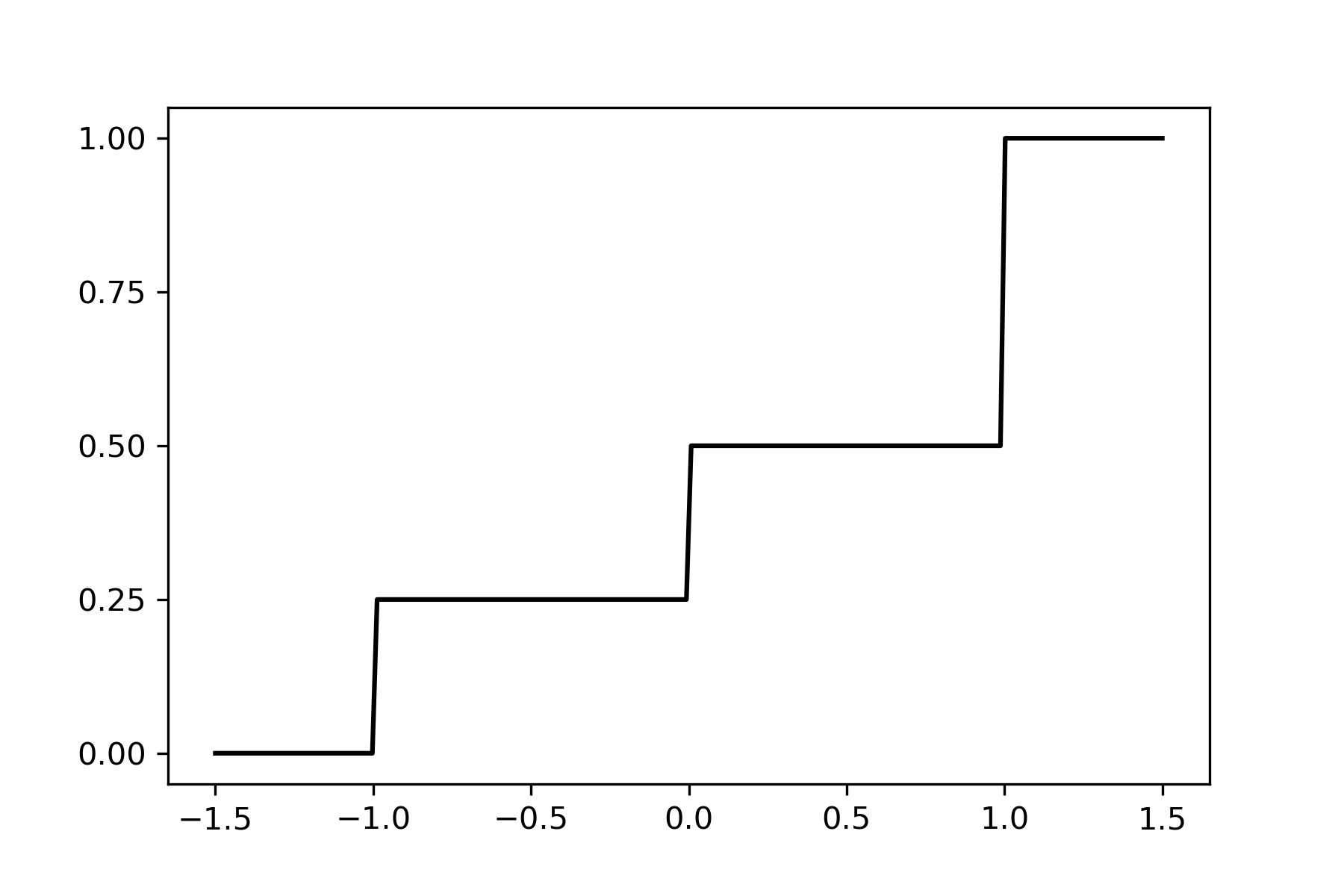}
    \end{center}
    \caption{The asymptotic predictive distributions produced by SVAPS
      when $A_0<A_1$
      for $x=0$ (left panel) and $x=1$ (right panel).}
    \label{fig:01}
  \end{figure}

  If $A_0<A_1$, the predictive distributions are as shown in Figure~\ref{fig:01}
  (the weights for $-1$, $0$, and $1$ are $1/4$, $3/4$, and $0$, respectively, when $x=0$,
  and $1/4$, $1/4$, and $1/2$, respectively, when $x=1$).
  For a continuous $f:\R\to[0,1]$ satisfying
  \begin{equation*}
    f(u)
    =
    \begin{cases}
      1 & \text{if $y\le-0.6$}\\
      0 & \text{if $y\ge-0.4$}
    \end{cases}
  \end{equation*}
  we will still have \eqref{eq:f-1} and \eqref{eq:f-2}.
\end{proof}

\begin{remark}
  In the proof of Proposition~\ref{prop:problem}
  we checked that the SVAPS are not universal
  directly, using the definition \eqref{eq:Belyaev}.
  This shows in their inferior CRPS,
  as in Example~\ref{ex:problem}.
  If $A_1=A_0$, we are in the situation of Example~\ref{ex:problem}.
  If $A_1<A_0$,
  the expected CRPS is $5/16$, which exceeds the ideal value $1/4$.
  And if $A_0<A_1$,
  the expected CRPS is also $5/16$.
\end{remark}

The asymptotic problem of non-universality for SVAPS can be avoided
by modifying, for each $y\in\R$,
the data sequence (including the training sequence proper) as follows:
replace each $y_i$ by $\II_{\{y_i\le y\}}$.
There are, however, two problems with this procedure:
\begin{itemize}
\item
  Loss of computational efficiency;
  now processing even moderately large datasets becomes infeasible.
\item
  Loss of predictive efficiency for small samples
  as now the labels become less informative
  (only taking values 0 or 1).
\end{itemize}

In this section we have only discussed SVAPS in detail,
but the same argument shows that cross-Venn--Abers predictive systems (defined in a natural way)
inherit the lack of universality.

\section{Conclusion}
\label{sec:conclusion}

In this paper we have given definitions and described ways of computing
split and cross-conformal predictive distributions.
We have studied their empirical performance
using five benchmark datasets and three underlying algorithms.
Cross-conformal predictive distributions are more efficient and,
in their non-randomized version,
sometimes closer to being valid.
It would be interesting to check the validity of our conclusions on a wider range of datasets
and underlying algorithms.

The specific split and cross-conformity measures used in this paper,
all of which have the form \eqref{eq:example-special},
are not fully adaptive, as discussed in Section~\ref{sec:univ},
whereas in general SCPS and CCPS are universal
(unlike SVAPS, as pointed out in Section~\ref{sec:VA}).
Replacing \eqref{eq:example-special} by the more general \eqref{eq:example}
somewhat improves the attainable flexibility,
but designing fully flexible cross-conformal predictive systems
based on efficient non-parametric predictive systems,
such as the Nadaraya--Watson system \eqref{eq:NW},
appears to us a particularly interesting direction of further research.

\subsection*{Acknowledgments}

Thanks to the AstraZeneca team (first of all Claus Bendtsen and Ola Engkvist) for useful discussions
and to Wang Di of Tianjin University for correcting an error.
We are very grateful to the three anonymous referees of the conference version of this paper
for their thoughtful comments;
because of time limitations we could not fully take some of them into account
in the conference version but have done it in this journal version.
Further comments by the three anonymous referees of the journal version are also gratefully appreciated;
they have led to important improvements (such as performing more extensive computational experiments).

\textbf{Funding:}
This work was supported by the EU Horizon 2020 Research and Innovation programme
[grant number 671555];
AstraZeneca [grant number R10911, ``Machine Learning for Chemical Synthesis''];
Leverhulme Magna Carta Doctoral Centre at Royal Holloway, University of London;
and NHS England supported by Innovate UK
[Technology Integrated Health Management project
awarded to the School of Mathematics and Information Security
at Royal Holloway, University of London].


\begin{thebibliography}{10}

\bibitem{Bala/etal:2014}
Vineeth~N. Balasubramanian, Shen-Shyang Ho, and Vladimir Vovk, editors.
\newblock {\em Conformal Prediction for Reliable Machine Learning: Theory,
  Adaptations, and Applications}.
\newblock Elsevier, Amsterdam, 2014.

\bibitem{Belyaev/Sjostedt:2000}
Yuri Belyaev and Sara {Sj\"ostedt--de Luna}.
\newblock Weakly approaching sequences of random distributions.
\newblock {\em Journal of Applied Probability}, 37:807--822, 2000.

\bibitem{Carlsson/etal:2014}
Lars Carlsson, Martin Eklund, and Ulf Norinder.
\newblock Aggregated conformal prediction.
\newblock In Lazaros Iliadis, Ilias Maglogiannis, Harris Papadopoulos, Spyros
  Sioutas, and Christos Makris, editors, {\em AIAI Workshops, COPA 2014},
  volume 437 of {\em IFIP Advances in Information and Communication
  Technology}, pages 231--240, Berlin, 2014. Springer.

\bibitem{UCI:2017}
Dua Dheeru and Efi Karra~Taniskidou.
\newblock {UCI} machine learning repository, 2017.

\bibitem{Gneiting/Katzfuss:2014}
Tilmann Gneiting and Matthias Katzfuss.
\newblock Probabilistic forecasting.
\newblock {\em Annual Review of Statistics and Its Application}, 1:125--151,
  2014.

\bibitem{Linusson/etal:2017}
Henrik Linusson, Ulf Norinder, Henrik Bostr\"om, Ulf Johansson, and Tuve
  L\"ofstr\"om.
\newblock On the calibration of aggregated conformal predictors.
\newblock {\em Proceedings of Machine Learning Research}, 60:154--173, 2017.
\newblock {COPA} 2017.

\bibitem{Nadaraya:1964}
Elizbar~A. Nadaraya.
\newblock On estimating regression.
\newblock {\em Theory of Probability and its Applications}, 9:141--142, 1964.

\bibitem{Nouretdinov/etal:2018COPA}
Ilia Nouretdinov, Denis Volkhonskiy, Pitt Lim, Paolo Toccaceli, and Alexander
  Gammerman.
\newblock Inductive {V}enn--{A}bers predictive distribution.
\newblock {\em Proceedings of Machine Learning Research}, 91:15--36, 2018.
\newblock {COPA} 2018.

\bibitem{Rosenblatt:1969}
Murray Rosenblatt.
\newblock Conditional probability density and regression estimators.
\newblock In Paruchuri~R. Krishnaiah, editor, {\em Multivariate Analysis II},
  pages 25--31. Academic Press, New York, 1969.

\bibitem{Schweder/Hjort:2016}
Tore Schweder and Nils~L. Hjort.
\newblock {\em Confidence, Likelihood, Probability: Sta\-tis\-ti\-cal Inference
  with Confidence Distributions}.
\newblock Cambridge University Press, Cambridge, 2016.

\bibitem{Shen/etal:2017}
Jieli Shen, Regina Liu, and Minge Xie.
\newblock Prediction with confidence---a general framework for predictive
  inference.
\newblock {\em Journal of Statistical Planning and Inference}, 195:126--140,
  2018.

\bibitem{Vovk:2015-cross}
Vladimir Vovk.
\newblock Cross-conformal predictors.
\newblock {\em Annals of Mathematics and Artificial Intelligence}, 74:9--28,
  2015.

\bibitem{Vovk:2019COPA}
Vladimir Vovk.
\newblock Universally consistent conformal predictive distributions.
\newblock {\em Proceedings of Machine Learning Research}, 105:105--122, 2019.
\newblock {COPA} 2019.

\bibitem{Vovk/Bendtsen:2018}
Vladimir Vovk and Claus Bendtsen.
\newblock Conformal predictive decision making.
\newblock {\em Proceedings of Machine Learning Research}, 91:52--62, 2018.
\newblock {COPA} 2018.

\bibitem{Vovk/etal:2005book}
Vladimir Vovk, Alex Gammerman, and Glenn Shafer.
\newblock {\em Algorithmic Learning in a Random World}.
\newblock Springer, New York, 2005.

\bibitem{Vovk/etal:2018Braverman}
Vladimir Vovk, Ilia Nouretdinov, Valery Manokhin, and Alex Gammerman.
\newblock Conformal predictive distributions with kernels.
\newblock In Lev Rozonoer, Boris Mirkin, and Ilya Muchnik, editors, {\em
  Braverman's Readings in Machine Learning: Key Ideas from Inception to Current
  State}, volume 11100, pages 103--121. Springer, Cham, Switzerland, 2018.

\bibitem{Vovk/etal:2018COPA}
Vladimir Vovk, Ilia Nouretdinov, Valery Manokhin, and Alex Gammerman.
\newblock Cross-conformal predictive distributions.
\newblock {\em Proceedings of Machine Learning Research}, 91:37--51, 2018.
\newblock {COPA} 2018. To appear in \emph{Neurocomputing}.

\bibitem{Vovk/Petej:2014UAI}
Vladimir Vovk and Ivan Petej.
\newblock Venn--{A}bers predictors.
\newblock In Nevin~L. Zhang and Jin Tian, editors, {\em Proceedings of the
  Thirtieth Conference on Uncertainty in Artificial Intelligence}, pages
  829--838, Corvallis, OR, 2014. AUAI Press.

\bibitem{Vovk/etal:2015NIPS}
Vladimir Vovk, Ivan Petej, and Valentina Fedorova.
\newblock Large-scale probabilistic predictors with and without guarantees of
  validity.
\newblock In C.~Cortes, N.~D. Lawrence, D.~D. Lee, M.~Sugiyama, and R.~Garnett,
  editors, {\em Advances in Neural Information Processing Systems 28}, pages
  892--900. Curran Associates, 2015.

\bibitem{Vovk/etal:arXiv1902}
Vladimir Vovk, Ivan Petej, Paolo Toccaceli, and Alex Gammerman.
\newblock Conformal calibrators.
\newblock Technical Report {\tt arXiv:1902.06579} [cs.LG], {\tt arXiv.org}
  e-Print archive, February 2019.

\bibitem{Vovk/etal:2019ML}
Vladimir Vovk, Jieli Shen, Valery Manokhin, and Minge Xie.
\newblock Nonparametric predictive distributions based on conformal prediction.
\newblock {\em Machine Learning}, 108:445--474, 2019.
\newblock {COPA} 2017 Special Issue.

\bibitem{Vovk:arXiv1212}
Vladimir Vovk and Ruodu Wang.
\newblock Combining p-values via averaging.
\newblock Technical Report
  \href{https://arxiv.org/abs/1212.4966}{arXiv:1212.4966 [math.ST]},
  \href{https://arxiv.org}{arXiv.org} e-Print archive, October 2019.
\newblock Journal version: \emph{Biometrika} (to appear).

\bibitem{Watson:1964}
Geoffrey~S. Watson.
\newblock Smooth regression analysis.
\newblock {\em Sankhy\=a A}, 26:359--372, 1964.

\end{thebibliography}
\end{document}